\documentclass[english]{article}
\pdfoutput=1
\usepackage[T1]{fontenc}
\usepackage[latin9]{inputenc}
\usepackage{geometry}
\usepackage{float}
\usepackage{amsthm}
\usepackage{amsmath}
\usepackage{amssymb}
\usepackage{graphicx}
\usepackage[numbers]{natbib}

\makeatletter

\providecommand{\tabularnewline}{\\}
\newcommand{\lyxdot}{.}

\floatstyle{ruled}
\newfloat{algorithm}{tbp}{loa}
\providecommand{\algorithmname}{Algorithm}
\floatname{algorithm}{\protect\algorithmname}

\theoremstyle{plain}
\newtheorem{thm}{\protect\theoremname}
  \theoremstyle{plain}
  \newtheorem{lem}[thm]{\protect\lemmaname}
  \theoremstyle{definition}
  \newtheorem{example}[thm]{\protect\examplename}

\usepackage{proceed2e}
\usepackage{algorithm}
\usepackage{algpseudocode}
\usepackage{times}
\usepackage{titlesec}
\titlespacing*\section{0pt}{4pt plus 2pt minus 2pt}{4pt plus 2pt minus 2pt}
\titlespacing*\subsection{0pt}{4pt plus 2pt minus 2pt}{4pt plus 2pt minus 2pt}
\titlespacing*\subsubsection{0pt}{4pt plus 2pt minus 2pt}{4pt plus 2pt minus 2pt}
\usepackage[font=small]{caption}

\makeatother

\usepackage{babel}
  \providecommand{\examplename}{Example}
  \providecommand{\lemmaname}{Lemma}
\providecommand{\theoremname}{Theorem}

\begin{document}

\title{Lifted Tree-Reweighted Variational Inference%
\thanks{$\ \ $Main article appeared in UAI 2014. This version also includes
the supplementary material.%
}}

\author{Hung Hai Bui \\
Natural Language Understanding Lab\\
Nuance Communications\\
Sunnyvale, CA, USA\\
bui.h.hung@gmail.com\\
\And Tuyen N. Huynh\\
John von Neumann Institute\\
Vietnam National University\\
 Ho Chi Minh City\\
tuyen.huynh@jvn.edu.vn\\
\And David Sontag\\
Courant Institute of Mathematical Sciences\\
New York University\\
dsontag@cs.nyu.edu}

\maketitle
\global\long\def\Int{\mathbb{N}}

\global\long\def\Real{\mathbb{R}}

\global\long\def\idf#1{\mathbb{I}\left\{  #1\right\}  }

\global\long\def\symg{\mathbb{S}}

\global\long\def\symgn{\symg_{n}}

\global\long\def\symgparti#1{\symg^{c}(#1)}

\global\long\def\Gg{\mathbb{G}}

\global\long\def\intlist#1{\{1,2,\ldots,#1\}}

\global\long\def\subgroup{\le}

\global\long\def\idg{\mathbf{1}}

\global\long\def\aut{\pi}

\global\long\def\Aut{\mathbb{A}}

\global\long\def\toorbit{\rho}

\global\long\def\orbit{\mathrm{orb}}

\global\long\def\Orbit{\mathrm{Orb}}

\global\long\def\orbitfont#1{{\bf #1}}

\global\long\def\vorb{\orbitfont v}

\global\long\def\eorb{\orbitfont e}

\global\long\def\aorb{\orbitfont a}

\global\long\def\orbitof#1{\overline{#1}}

\global\long\def\stab{\mathrm{Stab}}

\global\long\def\cluster{c}

\global\long\def\Cluster{\mathcal{C}}

\global\long\def\Clustersize#1{\Cluster_{n}^{#1}}

\global\long\def\tup{\tau}

\global\long\def\Tup{\mathcal{T}}

\global\long\def\Tupsize#1{\Tup_{n}^{#1}}

\global\long\def\tupset{\mathrm{set}}

\global\long\def\permvec#1#2{#1_{#2(1)},\ldots,#1_{#2(n)}}

\global\long\def\bp{\mathbf{p}}

\global\long\def\feat{\phi}

\global\long\def\nrfeat{\mathrm{f}}

\global\long\def\Feat{\Phi}

\global\long\def\para{\theta}

\global\long\def\Para{\Theta}

\global\long\def\mean{\mu}

\global\long\def\bmean{\boldsymbol{\mu}}

\global\long\def\Mean{\mathcal{M}}

\global\long\def\outbound{\text{\text{OUT}ER}}

\global\long\def\Local{\text{LOCAL}}

\global\long\def\Cyc{\text{CYCLE}}

\global\long\def\parti{\Delta}

\global\long\def\Index{\mathcal{I}}

\global\long\def\vIndex{\mathcal{V}}

\global\long\def\Ee{\mathcal{E}}

\global\long\def\domain{\mathcal{X}}

\global\long\def\Expt{\mathbb{E}}

\global\long\def\meanmap{\mathbf{m}}

\global\long\def\shatter{\varphi}

\global\long\def\Pot{\Psi}

\global\long\def\pot{\psi}

\global\long\def\naut{\aut}

\global\long\def\paut{\gamma}

\global\long\def\ord{\lambda}

\global\long\def\argl{\mathrm{argl}}

\global\long\def\args{scope}

\global\long\def\argn{\eta}

\global\long\def\argnall#1{|\argn_{#1}|}

\global\long\def\fgraph{\mathcal{F}}

\global\long\def\narg{\kappa}

\global\long\def\cgraph{\mathcal{C}}

\global\long\def\M{\mathcal{M}}
\global\long\def\x{\mathbf{x}}
\global\long\def\E{\text{E}}
\global\long\def\S{\mathcal{S}}
\global\long\def\b{\mathbf{b}}
\global\long\def\R{\mathbb{R}}
\global\long\def\N{\mathbb{N}}

\global\long\def\lmean{\bar{\mean}}

\global\long\def\lparti{\bar{\parti}}

\global\long\def\expandmat{D}

\global\long\def\pAut{\Aut^{p}}

\global\long\def\simparti#1{\stackrel{#1}{\sim}}

\global\long\def\relint{\text{ri\,}}

\global\long\def\riMean{\relint\Mean}

\global\long\def\chull{\text{conv }}

\global\long\def\FF{\mathcal{F}}

\global\long\def\gr{\mathfrak{G}}

\global\long\def\gm{\mathcal{G}}

\global\long\def\gcol{\text{col}}

\global\long\def\formula{F}

\global\long\def\ldom{\mathcal{D}}

\global\long\def\lconsts{\ldom_{0}}

\global\long\def\ldomnovel{\ldom_{*}}

\global\long\def\Ground{\text{Gr}}

\global\long\def\GroundFormula{\text{GrF}}

\global\long\def\model{\omega}

\global\long\def\rename{r}

\global\long\def\onrfeat{\nrfeat^{o}}

\global\long\def\ofeatunary#1#2{\feat_{#1:#2}^{o}}

\global\long\def\ofeatbinary#1#2#3#4{\feat_{\left\{  #1:#3,#2:#4\right\}  }^{o}}

\global\long\def\oMean{\Mean^{o}}

\global\long\def\opara{\para^{o}}

\global\long\def\omean{\mean^{o}}

\global\long\def\oparaunary#1#2{\para_{#1:#2}^{o}}

\global\long\def\oparabinary#1#2#3#4{\para_{\left\{  #1:#3,#2:#4\right\}  }^{o}}

\global\long\def\omeanunary#1#2{\mean_{#1:#2}^{o}}

\global\long\def\omeanbinary#1#2#3#4{\mean_{\left\{  #1:#3,#2:#4\right\}  }^{o}}

\global\long\def\pmeanunary#1#2{\tau_{#1:#2}}

\global\long\def\pmeanbinary#1#2#3#4{\tau_{\left\{  #1:#3,#2:#4\right\}  }}

\global\long\def\liftpmeanunary#1#2{\bar{\tau}_{#1:#2}}

\global\long\def\liftpmeanbinary#1#2#3{\bar{\tau}_{#1:#2#3}}

\global\long\def\oFF{\FF^{o}}

\global\long\def\oIndex{\Index^{o}}

\global\long\def\pair#1#2{#1\text{:}#2}

\global\long\def\paracollection{\stackrel{\rightarrow}{\theta}}

\global\long\def\stp{\mathbb{T}}

\global\long\def\ea{\rho}

\global\long\def\ead{\stackrel{\rightarrow}{\rho}}

\global\long\def\arrow#1{\stackrel{\rightarrow}{#1}}

\global\long\def\mappair#1#2{#1\text{:}#2}

\global\long\def\pairignore#1{\mappair{#1}*}

\global\long\def\itpair{\mappair it}

\global\long\def\jhpair{\mappair jh}

\global\long\def\extorb{\mathcal{E}}

\global\long\def\Aapprox{B}

\global\long\def\negEntApprox{\Aapprox^{*}}

\global\long\def\bu{\mathbf{u}}

\global\long\def\be{\mathbf{e}}

\global\long\def\choose#1#2{\left(\stackrel{#1}{#2}\right)}

\global\long\def\spt{\mathbb{T}}

\global\long\def\symsub#1#2{#1_{[#2]}}

\begin{abstract}
We analyze variational inference for highly symmetric graphical models
such as those arising from first-order probabilistic models. We first
show that for these graphical models, the tree-reweighted variational
objective lends itself to a compact lifted formulation which can be
solved much more efficiently than the standard TRW formulation for
the ground graphical model. Compared to earlier work on lifted belief
propagation, our formulation leads to a convex optimization problem
for lifted marginal inference and provides an upper bound on the partition
function. We provide two approaches for improving the lifted TRW upper
bound. The first is a method for efficiently computing maximum spanning
trees in highly symmetric graphs, which can be used to optimize the
TRW edge appearance probabilities. The second is a method for tightening
the relaxation of the marginal polytope using lifted cycle inequalities
and novel exchangeable cluster consistency constraints.\end{abstract}

\section{Introduction}

Lifted probabilistic inference focuses on exploiting symmetries in
probabilistic models for efficient inference \cite{braz05lifted,bui12lifted,bui13uai,gogate11probabilistic,milch&al08,niepert12uai,singla08lifted}.
Work in this area has demonstrated the possibility to perform very
efficient inference in highly-connected, large tree-width, but \emph{symmetric}
models, such as those arising in the context of relational (first-order)
probabilistic models and exponential family random graphs \cite{robins2007introduction}.
These models also arise frequently in probabilistic programming languages,
an area of increasing importance as demonstrated by DARPA's PPAML
program (Probabilistic Programming for Advancing Machine Learning).

Even though lifted inference can sometimes offer order-of-magnitude
improvement in performance, approximation is still necessary. A topic
of particular interest is the interplay between lifted inference and
variational approximate inference. Lifted loopy belief propagation
(LBP) \cite{jaimovich07uai,singla08lifted} was one of the first attempts
at exploiting symmetry to speed up loopy belief propagation; subsequently,
counting belief propagation (CBP) \cite{kersting09counting} provided
additional insights into the nature of symmetry in BP. Nevertheless,
these work were largely procedural and specific to the choice of message-passing
algorithm (in this case, loopy BP). More recently, Bui et al., \cite{bui13uai}
proposed a general framework for lifting a broad class of convex variational
techniques by formalizing the notion of symmetry (defined via automorphism
groups) of graphical models and the corresponding variational optimization
problems themselves, independent of any specific methods or solvers. 

Our goal in this paper is to extend the lifted variational framework
in \cite{bui13uai} to address the important case of approximate marginal
inference. In particular, we show how to lift the tree-reweighted
(TRW) convex formulation of marginal inference \cite{wainwright03treebased}.
As far as we know, our work presents the first lifted \emph{convex}
variational marginal inference, with the following benefits over previous
work: (1) a lifted convex upper bound of the log-partition function,
(2) a new tightening of the relaxation of the lifted marginal polytope
exploiting exchangeability, and (3) a convergent inference algorithm.
We note that convex upper bounds of the log-partition function immediately
lead to concave lower bounds of the log-likelihood which can serve
as useful surrogate loss functions in learning and parameter estimation
\cite{ChenLearning,jaimovich07uai}.

To achieve the above goal, we first analyze the symmetry of the TRW
log-partition and entropy bounds. Since TRW bounds depend on the choice
of the edge appearance probabilities $\ea$, we prove that the quality
of the TRW bound is not affected if one only works with suitably symmetric
$\ea$. Working with symmetric $\ea$ gives rise to an explicit lifted
formulation of the TRW optimization problem that is equivalent but
much more compact. This convex objective function can be convergently
optimized via a Frank-Wolfe (conditional gradient) method where each
Frank-Wolfe iteration solves a lifted MAP inference problem. We then
discuss the optimization of the edge-appearance vector $\ea$, effectively
yielding a lifted algorithm for computing maximum spanning trees in
symmetric graphs. 

As in Bui et al.'s framework, our work can benefit from any tightening
of the local polytope such as the use of cycle inequalities \cite{barahonacut,sontag}.
In fact, each method for relaxing the marginal polytope immediately
yields a variant of our algorithm. Notably, in the case of exchangeable
random variables, radically sharper tightening (sometimes even exact
characterization of the lifted marginal polytope) can be obtained
via a set of simple and elegant linear constraints which we call \emph{exchangeable
polytope constraints}. We provide extensive simulation studies comparing
the behaviors of different variants of our algorithm with exact inference
(when available) and lifted LBP demonstrating the advantages of our
approach. The supplementary material \cite{bui2014liftedtrw} provides
additional proof and algorithm details.

\section{Background}

We begin by reviewing variational inference and the tree-reweighted
(TRW) approximation. We focus on inference in Markov random fields,
which are distributions in the exponential family given by $\Pr(x;\theta)=\exp\,\left\{ \left\langle \Phi(x),\theta\right\rangle -A(\theta)\right\} $,
where $A(\theta)$ is called the \emph{log-partition function} and
serves to normalize the distribution. We assume that the random variables
$x\in\domain^{n}$ are discrete-valued, and that the features $(\Feat_{i}),\ i\in\Index$
factor according to the graphical model structure $\gm$; $\Phi$
can be non-pairwise and is assumed to be overcomplete. This paper
focuses on the inference tasks of estimating the marginal probabilities
$p(x_{i})$ and approximating the log-partition function. Throughout
the paper, the domain $\domain$ is the binary domain $\{0,1\}$,
however, except for the construction of exchangeable polytope constraints
in Section \ref{sec:exch-constraints}, this restriction is not essential. 

Variational inference approaches view the log-partition function as
a convex optimization problem over the marginal polytope $A(\para)=\sup_{\mean\in\Mean(\gm)}\langle\mean,\theta\rangle-A^{*}(\mean)$
and seek tractable approximations of the negative entropy $A^{*}$
and the marginal polytope $\Mean$ \cite{wainwright2008graphical}.
Formally, $-A^{*}(\mu)$ is the entropy of the maximum entropy distribution
with moments $\mu$. Observe that $-A^{*}(\mu)$ is upper bounded
by the entropy of the maximum entropy distribution consistent with
any subset of the expected sufficient statistics $\mu$. To arrive
at the TRW approximation \cite{martinUpper}, one uses a subset given
by the pairwise moments of a spanning tree%
\footnote{If the original model contains non-pairwise potentials, they can be
represented as cliques in the graphical model, and the bound based
on spanning trees still holds.%
}. Hence for any distribution $\ea$ over spanning trees, an upper
bound on $-A^{*}$ is obtained by taking a convex combination of tree
entropies $-\negEntApprox(\tau,\ea)=\sum_{s\in V(G)}H(\tau_{s})-\sum_{e\in E(G)}I(\tau_{e})\ea_{e}$.
Since $\ea$ is a distribution over spanning trees, it must belong
to the spanning tree polytope $\stp(\gm)$ with $\ea_{e}$ denoting
the edge appearance probability of $e$. Combined with a relaxation
of the marginal polytope $\outbound\supset\Mean$, an upper bound
$B$ of the log-partition function is obtained: 
\begin{equation}
A(\para)\le\Aapprox(\para,\ea)=\sup_{\tau\in\outbound(\gm)}\left\langle \tau,\theta\right\rangle -\negEntApprox(\tau,\ea)\label{eq:trw-primal}
\end{equation}
We note that $\negEntApprox$ is linear w.r.t. $\ea$, and for $\ea\in\stp(G)$,
$\negEntApprox$ is convex w.r.t. $\tau$. On the other hand, $B$
is convex w.r.t. $\ea$ and $\para$.

The optimal solution $\tau^{*}(\ea,\para)$ of the optimization problem
(\ref{eq:trw-primal}) can be used as an approximation to the mean
parameter $\mean(\para)$. Typically, the local polytope $\Local$
given by pairwise consistency constraints is used as the relaxation
$\outbound$; in this paper, we also consider tightening of the local
polytope. 

Since (\ref{eq:trw-primal}) holds with any edge appearance $\ea$
in the spanning tree polytope $\stp$, the TRW bound can be further
improved by optimizing $\ea$
\begin{equation}
\inf_{\ea\in\mathbb{T}(G)}\Aapprox(\theta,\ea)\label{eq:optimize-edge-appearance}
\end{equation}
The resulting $\ea^{*}$ is then plugged into (\ref{eq:trw-primal})
to find the marginal approximation. In practice, one might choose
to work with some fixed choice of $\ea$, for example the uniform
distribution over all spanning trees. \cite{JancsaryM11} proposed
using the most uniform edge-weight $\arg\inf_{\ea\in\stp(G)}\sum_{e\in E}(\ea_{e}-\frac{|V|-1}{|E|})^{2}$
which can be found via conditional gradient where each direction-finding
step solves a maximum spanning tree problem.

Several algorithms have been proposed for optimizing the TRW objective
(\ref{eq:trw-primal}) given fixed edge appearance probabilities.
\cite{wainwright2008graphical} derived the tree-reweighted belief
propagation algorithm from the fixed point conditions. \cite{GJ07}
show how to solve the dual of the TRW objective, which is a geometric
program. Although this algorithm has the advantage of guaranteed convergence,
it is non-trivial to generalize this approach to use tighter relaxations
of the marginal polytope, which we show is essential for lifted inference.
\cite{JancsaryM11} use an explicit set of spanning trees and then
use dual decomposition to solve the optimization problem. However,
as we show in the next section, to maintain symmetry it is essential
that one \emph{not} work directly with spanning trees but rather use
symmetric edge appearance probabilities. \cite{sontag} optimize TRW
over the local and cycle polytopes using a Frank-Wolfe (conditional
gradient) method, where each iteration requires solving a linear program.
We follow this latter approach in our paper.

To optimize the edge appearance in (\ref{eq:optimize-edge-appearance}),
\cite{martinUpper} proposed using conditional gradient. They observed
that $\frac{\partial B(\para,\rho)}{\partial\ea_{e}}=-\frac{\partial B^{*}(\tau^{*},\ea)}{\partial\ea_{e}}=-I(\tau_{e}^{*})$
where $\tau^{*}$ is the solution of (\ref{eq:trw-primal}). The direction-finding
step in conditional gradient reduces to solving $\sup_{\ea\in\stp}\langle\ea,I\rangle$,
again equivalent to finding the maximum spanning tree with edge mutual
information $I(\tau_{e}^{*})$ as weights. We discuss the corresponding
lifted problem in section \ref{sec:lifted-mst}.

\section{Lifted Variational Framework}

We build on the key element of the lifted variational framework introduced
in \cite{bui13uai}. The automorphism group of a graphical model,
or more generally, an exponential family is defined as the group $\Aut$
of permutation pairs $(\naut,\paut)$ where $\naut$ permutes the
set of variables and $\paut$ permutes the set of features in such
a way that they preserve the feature function: $\Feat^{\paut^{-1}}(x^{\naut})=\Feat(x)$.
Note that this construction of $\Aut$ is entirely based on the structure
of the model and does not depend on the particular choice of the model
parameters; nevertheless the group stabilizes%
\footnote{Formally, $\Gg$ stabilizes $x$ if $x^{g}=x$ for all $g\in\Gg$.%
} (preserves) the key characteristics of the exponential family such
as the marginal polytope $\Mean$, the log-partition $A$ and entropy
$-A^{*}$. 

As shown in \cite{bui13uai} the automorphism group is particularly
useful for exploiting symmetries when parameters are tied. For a given
parameter-tying partition $\parti$ such that $\para_{i}=\para_{j}$
for $i,j$ in the same cell%
\footnote{If $\parti=\{\parti_{1}\ldots\parti_{K}\}$ is a partition of $S$,
then each subset $\parti_{k}\subset S$ is called a cell.%
} of $\parti$, the group $\Aut$ gives rise to a subgroup called the
lifting group $\Aut_{\parti}$ that stabilizes the tied-parameter
vector $\para$ as well as the exponential family. The orbit partition
of the the lifting group can be used to formulate equivalent but more
compact variational problems. More specifically, let $\shatter=\shatter(\parti)$
be the orbit partition induced by the lifting group on the feature
index set $\Index=\{1\ldots m\}$, let $ $$\Real_{[\varphi]}^{m}$
denote the symmetrized subspace $\{r\in\Real^{m}\text{ s.t. }r_{i}=r_{j}\ \forall i,j\text{ in the same cell of }\varphi\}$
and define the lifted marginal polytope $\symsub{\Mean}{\shatter}$
as $\Mean\cap\Real_{[\varphi]}^{m}$, then (see Theorem 4 of \cite{bui13uai})
\begin{equation}
\sup_{\mean\in\Mean}\left\langle \para,\mean\right\rangle -A^{*}(\mean)=\sup_{\mean\in\symsub{\Mean}{\shatter}}\left\langle \para,\mean\right\rangle -A^{*}(\mean)\label{eq:lifted-mean}
\end{equation}

In practice, we need to work with convex variational approximations
of the LHS of (\ref{eq:lifted-mean}) where $\Mean$ is relaxed to
an outer bound $\outbound(\gm)$ and $A^{*}$ is approximated by a
convex function $\negEntApprox(\mean)$. We now state a similar result
for lifting general convex approximations.
\begin{thm}
\label{thm:equivalent}If $\negEntApprox(\mean)$ is convex and stabilized
by the lifting group $\Aut_{\parti}$, i.e., for all $(\naut,\paut)\in\Aut_{\parti}$,
$\negEntApprox(\mean^{\paut})=\negEntApprox(\mean)$, then $\varphi$
is the lifting partition for the approximate variational problem\textup{
\begin{equation}
\sup_{\mean\in\outbound(\gm)}\left\langle \para,\mean\right\rangle -\negEntApprox(\mean)=\sup_{\mean\in\symsub{\outbound}{\shatter}}\left\langle \para,\mean\right\rangle -\negEntApprox(\mean)\label{eq:equivalent}
\end{equation}
}

\emph{The importance of Theorem \ref{thm:equivalent} is that it shows
that it is equivalent to optimize over a subset of $\outbound(\gm)$
where pseudo-marginals in the same orbit are restricted to take the
same value. As we will show in Section \ref{sub:Lifted-TRW-Problems},
this will allow us to combine many of the terms in the objective,
which is where the computational gains will derive from. A sketch
of its proof is as follows. Consider a single pseudo-marginal vector
$\mu$. Since the objective value is the same for every $\mu^{\gamma}$
for }$(\naut,\paut)\in\Aut_{\parti}$ \emph{and since the objective
is concave, the }average\emph{ of these, $\frac{1}{|\Aut_{\parti}|}\sum_{(\naut,\paut)\in\Aut_{\parti}}\mathcal{\mu^{\gamma}}$,
must have at least as good of an objective value. Furthermore, note
that this averaged vector lives in the symmetrized subspace. Thus,
it suffices to optimize over }\textup{$\symsub{\outbound}{\shatter}$.}\end{thm}

\section{Lifted Tree-Reweighted Problem}

\subsection{Symmetry of TRW Bounds}

We now show that Theorem 1 can be used to lift the TRW optimization
problem (\ref{eq:trw-primal}). Note that the applicability of Theorem
1 is not immediately obvious since $\negEntApprox$ depends on the
distribution over trees implicit in $\ea$. In establishing that the
condition in Theorem 1 holds, we need to be careful so that the choice
of the distribution over trees $\rho$ does not destroy the symmetry
of the problem.

The result below ensures that with no loss in optimality, $\rho$
can be assumed to be suitably symmetric. More specifically, let $ $$\varphi^{E}=\varphi^{E}(\parti)$
be the set of $\gm$'s edge orbits induced by the action of the lifting
group $\Aut_{\parti}$; the edge-weights $\rho_{e}$ for every $e$
in the same edge orbits can be constrained to be the same, i.e. $\rho$
can be restricted to $\mathbb{T}_{[\varphi^{E}]}$.
\begin{thm}
\label{thm:symmetric_rho}For any $\ea\in\stp$, there exists a symmetrized
$\hat{\rho}\in\symsub{\stp}{\varphi^{E}}$ that yields at least as
good an upper bound, i.e. 
\[
B(\theta,\hat{\ea})\le B(\para,\ea)\ \ \forall\para\in\Para_{[\parti]}
\]
As a consequence, in optimizing the edge appearance, $\rho$ can be
restricted to the symmetrized spanning tree polytope $\stp_{[\varphi^{E}]}$
\[
\forall\para\in\Para_{[\parti]},\ \inf_{\rho\in\stp}\Aapprox(\para,\rho)=\inf_{\rho\in\symsub{\stp}{\varphi^{E}}}\Aapprox(\para,\rho)
\]
\end{thm}
\begin{proof}
Let $\ea$ be the argmin of the LHS, and define $\hat{\ea}=\frac{1}{|\Aut_{\parti}|}\sum_{\pi\in\Aut_{\parti}}\rho^{\pi}$
so that $\hat{\rho}\in\symsub{\stp}{\varphi^{E}}$. For all $(\pi,\gamma)\in\Aut_{\parti}$
and for all tied-parameter $\para\in\symsub{\Theta}{\parti}$, $\theta^{\pi}=\theta$,
so $B(\theta,\rho^{\pi})=B(\theta^{\naut},\rho^{\naut})$. By theorem
1 of \cite{bui13uai}, $\naut$ must be an automorphism of the graph
$\gm$. By lemma \ref{lem:symmetry-of-B} (see supplementary material),
$B(\theta^{\naut},\rho^{\naut})=\Aapprox(\para,\ea)$. Thus $B(\theta,\rho^{\pi})=B(\para,\ea)$.
$ $Since $B$ is convex w.r.t. $\ea$, by Jensen's inequality we
have that $ $$B(\theta,\hat{\ea})\le\frac{1}{|\Aut_{\parti}|}\sum_{\pi\in\Aut_{\parti}}B(\theta,\rho^{\pi})=B(\para,\ea).$
\end{proof}
Using a symmetric choice of $\rho$, the TRW bound $\negEntApprox$
then satisfies the condition of theorem \ref{thm:equivalent}, enabling
the applicability of the general lifted variational inference framework.
\begin{thm}
\label{thm:lifted-trw-equivalent}For a fixed $\ea\in\symsub{\stp}{\varphi^{E}}$,
$\varphi$ is the lifting partition for the TRW variational problem
\begin{equation}
\sup_{\tau\in\outbound(\gm)}\left\langle \tau,\theta\right\rangle -\negEntApprox(\tau,\ea)=\sup_{\tau\in\symsub{\outbound}{\shatter}}\left\langle \tau,\theta\right\rangle -\negEntApprox(\tau,\ea)\label{eq:lifted-trw}
\end{equation}

\end{thm}

\subsection{Lifted TRW Problems\label{sub:Lifted-TRW-Problems}}

\global\long\def\liftedB{\overline{\negEntApprox}}

We give the explicit lifted formulation of the TRW optimization problem
(\ref{eq:lifted-trw}). As in \cite{bui13uai}, we restrict $\tau$
to $\symsub{\outbound}{\shatter}$ by introducing the lifted variables
$\bar{\tau}_{j}$ for each cell $\varphi_{j}$, and for all $i\in\varphi_{j}$,
enforcing that $\tau_{i}=\bar{\tau}_{j}$. Effectively, we substitute
every occurrence of $\tau_{i}$, $i\in\varphi_{j}$ by $\bar{\tau}_{j}$;
in vector form, $\tau$ is substituted by $\expandmat\bar{\tau}$
where $D$ is the characteristic matrix of the partition $\varphi$:
$\expandmat_{ij}=1$ if $i\in\shatter_{j}$ and $0$ otherwise. This
results in the lifted form of the TRW problem
\begin{equation}
\sup_{D\bar{\tau}\in\outbound}\left\langle \bar{\tau},\bar{\theta}\right\rangle -\overline{\negEntApprox}(\bar{\tau},\bar{\ea})\label{eq:lifted-trw-tosolve}
\end{equation}
where $\bar{\para}=D^{\top}\para$; $\overline{}$$\liftedB$ is obtained
from $\negEntApprox$ via the above substitution; and $\bar{\ea}$
is the edge appearance per edge-orbit: for every edge orbit $\eorb$,
and for every edge $e\in\eorb$, $\ea_{e}=\bar{\ea}_{\eorb}$. Using
an alternative but equivalent form $\negEntApprox=-\sum_{v\in V}(1-\sum_{e\in Nb(v)}\ea_{e})H(\tau_{v})-\sum_{e\in E}\ea_{e}H(\tau_{e})$,
we obtain the following explicit form for 
\begin{eqnarray}
\liftedB(\bar{\tau},\bar{\ea}) & = & -\sum_{\vorb\in\bar{V}}\left(|\vorb|-\sum_{\eorb\in N(\vorb)}|\eorb|d(\vorb,\eorb)\bar{\ea}_{\eorb}\right)H(\bar{\tau}_{\vorb})\nonumber \\
 &  & -\sum_{\eorb\in\bar{E}}|\eorb|\bar{\ea}_{\eorb}H(\bar{\tau}_{\eorb})\label{eq:lifted-trw-explicit}
\end{eqnarray}
Intuitively, the above can be viewed as a combination of node and
edge entropies defined on nodes and edges of the lifted graph $\bar{\gm}$.
Nodes of $\bar{\gm}$ are the node orbits of $\gm$ while edges are
the edge-orbits of $\gm$. $\bar{\gm}$ is not a simple graph: it
can have self-loops or multi-edges between the same node pair (see
Fig. \ref{fig:lifted-graph}). We encode the incidence on this graph
as follows: $d(\vorb,\eorb)=0$ if $\vorb$ is not incident to $\eorb$,
$d(\vorb,\eorb)=1$ if $\vorb$ is incident to $\eorb$ and $\eorb$
is not a loop, $d(\vorb,\eorb)=2$ if $\eorb$ is a loop incident
to $\vorb$. The \emph{entropy} at the node orbit $\vorb$ is defined
as
\[
H(\bar{\tau}_{\vorb})=-\sum_{t\in\domain}\liftpmeanunary{\vorb}t\ln(\liftpmeanunary{\vorb}t)
\]
and the entropy at the edge orbit $\eorb$ is
\begin{eqnarray*}
H(\bar{\tau}_{\eorb}) & = & -\sum_{t,h\in\domain}\bar{\tau}_{\orbitof{\{e_{1}:t,e_{2}:h\}}}\ln(\bar{\tau}_{\orbitof{\{e_{1}:t,e_{2}:h\}}})
\end{eqnarray*}
where $\{e_{1},e_{2}\}$ for $e_{1},e_{2}\in V$ is a representative
(any element) of $\eorb$, $\{\pair{e_{1}}t,\pair{e_{2}}h\}$ is an
assignment of the ground edge $\{e_{1},e_{2}\}$, and $\overline{\{\pair{e_{1}}t,\pair{e_{2}}h\}}$
is the assignment orbit. As in \cite{bui13uai}, we write $\overline{\{\pair{e_{1}}t,\pair{e_{2}}t\}}$
as $\pair{\eorb}t$, and $ $for $t<h$, $\overline{\{\pair{e_{1}}t,\pair{e_{2}}h\}}$
as $\pair{\aorb}{(t,h)}$ where $\aorb$ is the arc-orbit $\orbitof{(e_{1},e_{2})}$.

When $\outbound$ is the local or cycle polytope, the constraints
$ $$D\bar{\tau}\in\outbound$ yield the lifted local (or cycle) polytope
respectively. For these constraints, we use the same form given in
\cite{bui13uai}. In section \ref{sec:exch-constraints}, we describe
a set of additional constraints for further tightening when some cluster
of nodes are exchangeable.

\begin{figure}
\centering{}\includegraphics[scale=0.4]{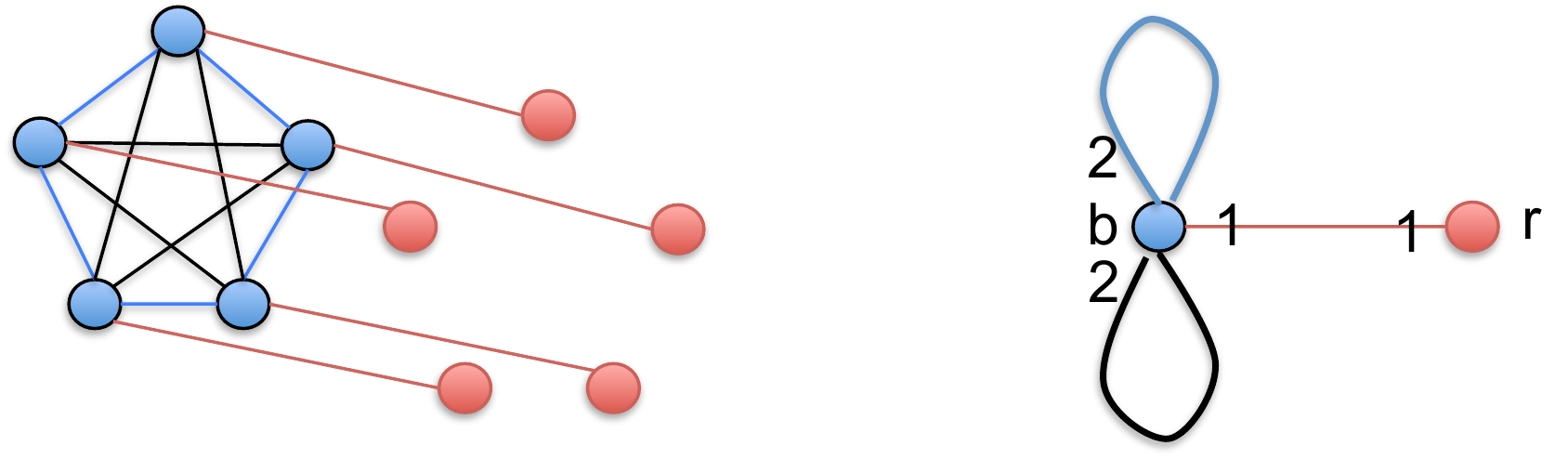}\caption{\label{fig:lifted-graph}Left: ground graphical model. Same colored
nodes and edges have the same parameters. Right: lifted graph showing
2 node orbits (\textbf{b} and \textbf{r}), and 3 edge orbits. Numbers
on the lifted graph representing the incidence degree $d(\vorb,\eorb)$
between an edge and a node orbit.}
\end{figure}

\textbf{Example.} Consider the MRF shown in Fig. \ref{fig:lifted-graph}
(left) with 10 binary variables that we denote $B_{i}$ (for the blue
nodes) and $R_{i}$ (for the red nodes). The node and edge coloring
denotes shared parameters. Let $\theta_{b}$ and $\theta_{r}$ be
the single-node potentials used for the blue and red nodes, respectively.
Let $\theta_{r_{e}}$ be the edge potential used for the red edges
connecting the blue and red nodes, $\theta_{b_{e}}$ for the edge
potential used for the blue edges $(B_{i},B_{i+1})$, and $\theta_{k_{e}}$
for the edge potential used for the black edges $(B_{i},B_{i+2})$.

There are two node orbits: ${\bf b}=\{B_{1},\ldots,B_{5}\}$ and ${\bf r}=\{R_{1},\ldots,R_{5}\}.$
There are three edge orbits: ${\bf r_{e}}$ for the red edges, ${\bf b_{e}}$
for the blue edges , and ${\bf k_{e}}$ for the black edges. The size
of the node and edge orbits are all 5 (e.g., $|{\bf b}|=|{\bf b_{e}}|=5$),
and $d({\bf b},{\bf r_{e}})=d({\bf r},{\bf r_{e}})=1$, $d({\bf b},{\bf b_{e}})=d({\bf b},{\bf k_{e}})=2$.
Suppose that $\rho$ corresponds to a uniform distribution over spanning
trees, which satisfies the symmetry needed by Theorem \ref{thm:symmetric_rho}.
We then have $\overline{\rho}_{{\bf r_{e}}}=1$ and $\overline{\rho}_{{\bf b_{e}}}=\overline{\rho}_{{\bf k_{e}}}=\frac{2}{5}$.
Putting all of this together, the lifted TRW entropy is given by
$\liftedB(\bar{\tau},\bar{\ea})=8H(\overline{\tau}_{{\bf b}})-5H(\overline{\tau}_{{\bf r_{e}}})-2H(\overline{\tau}_{{\bf b_{e}}})-2H(\overline{\tau}_{{\bf k_{e}}})$.
We illustrate the expansion of the entropy of the red edge orbit $H(\bar{\tau}_{{\bf r_{e}}})$.
This edge orbit has 2 corresponding arc-orbits: ${\bf rb_{a}}=\{(R_{i},B_{i})\}$
and $ $${\bf br_{a}}=\{(B_{i},R_{i})\}$. Thus, $H(\bar{\tau}_{{\bf r_{e}}})=-\bar{\tau}_{{\bf r_{e}}:00}\ln\bar{\tau}_{{\bf r_{e}}:00}-\bar{\tau}_{{\bf r_{e}}:11}\ln\bar{\tau}_{{\bf r_{e}}:11}-\bar{\tau}_{{\bf rb_{a}}:01}\ln\bar{\tau}_{{\bf rb_{a}}:01}-\bar{\tau}_{{\bf br_{a}}:01}\ln\bar{\tau}_{{\bf br_{a}}:01}$.

Finally, the linear term in the objective is given by $\left\langle \bar{\tau},\bar{\theta}\right\rangle =$$5\left\langle \bar{\tau}_{{\bf b}},\theta_{b}\right\rangle +5\left\langle \bar{\tau}_{{\bf r}},\theta_{r}\right\rangle +5\left\langle \bar{\tau}_{{\bf r_{e}}},\theta_{r_{e}}\right\rangle +5\left\langle \bar{\tau}_{{\bf b_{e}}},\theta_{b_{e}}\right\rangle +5\left\langle \bar{\tau}_{{\bf k_{e}}},\theta_{k_{e}}\right\rangle $
where, as an example, $\left\langle \bar{\tau}_{{\bf r_{e}}},\theta_{r_{e}}\right\rangle =\bar{\tau}_{{\bf r_{e}}:00}\theta_{r_{e},00}+\bar{\tau}_{{\bf r_{e}}:11}\theta_{r_{e},11}+\bar{\tau}_{{\bf {\bf br_{a}}}:01}\theta_{r_{e},01}+\bar{\tau}_{{\bf {\bf rb_{a}}}:01}\theta_{r_{e},10}$

\subsection{Optimization using Frank-Wolfe}

What remains is to describe how to optimize Eq. \ref{eq:lifted-trw-tosolve}.
Our lifted tree-reweighted algorithm is based on Frank-Wolfe, also
known as the conditional gradient method \cite{frank-wolfe56,ICML2013_jaggi13}.
First, we initialize with a pseudo-marginal vector corresponding to
the uniform distribution, which is guaranteed to be in the lifted
marginal polytope. Next, we solve the linear program whose objective
is given by the gradient of the objective Eq. \ref{eq:lifted-trw-tosolve}
evaluated at the current point, and whose constraint set is $\outbound$.
When using the lifted cycle relaxation, we solve this linear program
using a cutting-plane algorithm \cite{bui13uai,sontag}. We then perform
a line search to find the optimal step size using a golden section
search (a type of binary search that finds the maxima of a unimodal
function), and finally repeat using the new pseudo-marginal vector.
We warm start each linear program using the optimal basis found in
the previous run, which makes the LP solves extremely fast after the
first couple of iterations. Although we use a generic LP solver in
our experiments, it is also possible to use dual decomposition to
derive efficient algorithms  specialized to graphical models \cite{sontag08tightening}.

\section{Lifted Maximum Spanning Tree\label{sec:lifted-mst}}

Optimizing the TRW edge appearance probability $\ea$ requires finding
the maximum spanning tree (MST) in the ground graphical model. For
lifted TRW, we need to perform MST while using only information from
the node and edge orbits, without referring to the ground graph. In
this section, we present a lifted MST algorithm for symmetric graphs
which works at the orbit level.

Suppose that we are given a \emph{weighted} graph $(\mathcal{G},w)$,
its automorphism group $\Aut=Aut(\mathcal{G})$ and its node and edge
orbits. We aim to derive an algorithm analogous to the Kruskal's algorithm,
but with complexity depends only on the number of node/edge orbits
of $\mathcal{G}$. However, if the algorithm has to return an actual
spanning tree of $\mathcal{G}$ then clearly its complexity cannot
be less than $O(|V|).$ Instead, we consider an equivalent problem:
solving a linear program on the spanning-tree polytope
\begin{equation}
\sup_{\rho\in\mathbb{T}(\mathcal{G})}\left\langle \rho,w\right\rangle \label{eq:spanning-tree-lp}
\end{equation}
The same mechanism for lifting convex optimization problem (Lemma
1 in \cite{bui13uai}) applies to this problem. Let $\varphi^{E}$
be the edge orbit partition, then an equivalent lifted problem problem
is
\begin{equation}
\sup_{\rho\in\mathbb{T}_{[\varphi^{E}]}}\left\langle \rho,w\right\rangle \label{eq:spanning-tree-lp-sym}
\end{equation}
Since $\ea_{e}$ is constrained to be the same for edges in the same
orbit, it is now possible to solve (\ref{eq:spanning-tree-lp-sym})
with complexity depending only on the number of orbits. Any solution
$\ea$ of the LP (\ref{eq:spanning-tree-lp}) can be turned into a
solution $\bar{\rho}$ of (\ref{eq:spanning-tree-lp-sym}) by letting
$\bar{\rho}(\eorb)=\frac{1}{|\eorb|}\sum_{e'\in\eorb}\rho(e')$ .

\subsection{Lifted Kruskal's Algorithm}

The Kruskal's algorithm first sorts the edges according to their decreasing
weight. Then starting from an empty graph, at each step it greedily
attempts to add the next edge while maintaining the property that
the used edges form a forest (containing no cycle). The forest obtained
at the end of this algorithm is a maximum-weight spanning tree. 

Imagine how Kruskal's algorithm would operate on a weighted graph
$\mathcal{G}$ with non-trivial automorphisms. Let $\mathbf{e}_{1},\ldots,\mathbf{e}_{k}$
be the list of edge-orbits sorted in the order of decreasing weight
(the weights $w$ on all edges in the same orbit by definition must
be the same). The main question therefore is how many edges in each
edge-orbit $\mathbf{e}_{i}$ will be added to the spanning tree by
the Kruskal's algorithm. Let $\mathcal{G}_{i}$ be the subgraph of
$\gm$ formed by the set of all the edges and nodes in $\mathbf{e}_{1},\ldots\eorb_{i}$.
Let $V(\gm)$ and $C(\gm)$ denote the set of nodes and set of connected
components of a graph, respectively. Then (see the supplementary material
for proof)
\begin{lem}
\label{lem:lifted-mst}The number of edges in $\eorb_{i}$ appearing
in the MST found by the Kruskal's algorithm is $\delta_{V}^{(i)}-\delta_{C}^{(i)}$
where $\delta_{V}^{(i)}=|V(\mathcal{G}_{i})|-|V(\mathcal{G}_{i-1})|$
and $\delta_{C}^{(i)}=|C(\gm_{i})|-|C(\gm_{i-i})|$. Thus a solution
for the linear program (\ref{eq:spanning-tree-lp-sym}) is $\bar{\ea}(\eorb_{i})=\frac{\delta_{V}^{(i)}-\delta_{C}^{(i)}}{|\eorb_{i}|}$.
\end{lem}

\subsection{Lifted Counting of the Number of Connected Components}

We note that counting the number of nodes can be done simply by adding
the size of each node orbit. The remaining difficulty is how to count
the number of connected components of a given graph%
\footnote{Since we are only interested in connectivity in this subsection, the
weights of $\mathcal{G}$ play no role. Thus, orbits in this subsection
can also be generated by the automorphism group of the unweighted
version of $\mathcal{G}$. %
} $\mathcal{G}$ using only information at the orbit level. Let $\bar{\mathcal{G}}$$ $
be the lifted graph of $\gm$. Then (see supplementary material for
proof) 
\begin{lem}
\label{lem:counting-connected-component}If $\bar{\mathcal{G}}$ is
connected then all connected components of $\mathcal{G}$ are isomorphic.
Thus if furthermore $\gm'$ is a connected component of $\gm$ then
$|C(\gm)|=|V(\gm)|/|V(\gm')|$.
\end{lem}
To find just one connected component, we can choose an arbitrary node
$u$ and compute $\bar{\mathcal{G}}[u]$, the lifted graph fixing
$u$ (see section 8.1 in \cite{bui13uai}), then search for the connected
component in $ $$\bar{\mathcal{G}}[u]$ that contains $\{u\}$. 
Finally, if $\bar{\mathcal{G}}$ is not connected, we simply apply
lemma \ref{lem:counting-connected-component} for each connected component
of $\bar{\mathcal{G}}$.

The final lifted Kruskal's algorithm combines lemma \ref{lem:lifted-mst}
and \ref{lem:counting-connected-component} while keeping track of
the set of connected components of $\bar{\gm}_{i}$ incrementally.
The full algorithm is given in the supplementary material.

\section{Tightening via Exchangeable Polytope Constraints\label{sec:exch-constraints}}

One type of symmetry often found in first-order probabilistic models
are large sets of exchangeable random variables. In certain cases,
exact inference with exchangeable variables is possible via lifted
counting elimination and its generalization \cite{milch&al08,bui12lifted}.
The drawback of these exact methods is that they do not apply to many
models (e.g., those with transitive clauses). Lifted variational inference
methods do not have this drawback, however local and cycle relaxation
can be shown to be loose in the exchangeable setting, a potentially
serious limitation compared to earlier work. 

To remedy this situation, we now show how to take advantage of highly
symmetric subset of variables to tighten the relaxation of the lifted
marginal polytope.

\global\long\def\exchvars{\chi}

\global\long\def\config{\mathfrak{C}}

We call a set of random variables $\exchvars$ an \emph{exchangeable}
cluster iff $\exchvars$ can be arbitrary permuted while preserving
the probability model. Mathematically, the lifting group $\Aut_{\Delta}$
acts on $\exchvars$ and the image of the action is isomorphic to
$\symg(\exchvars)$, the symmetric group on $\exchvars$. The distribution
of the random variables in $\exchvars$ is also exchangeable in the
usual sense.

Our method for tightening the relaxation of the marginal polytope
is based on lift-and-project, wherein we introduce auxiliary variables
specifying the joint distribution of a large cluster of variables,
and then enforce consistency between the cluster distribution and
the pseudo-marginal vector \cite{sherali_adams90,sontag08tightening,wainwright2008graphical}.
In the ground model, one typically works with small clusters (e.g.,
triplets) because the number of variables grows exponentially with
cluster size. The key (and nice) difference in the lifted case is
that we can make use of very large clusters of highly symmetric variables:
while the grounded relaxation would clearly blow up, the corresponding
lifted relaxation can still remain compact.

Specifically, for an exchangeable cluster $\exchvars$ of arbitrary
size, one can add cluster consistency constraints for the entire cluster
and still maintain tractability. To keep the exposition simple, we
assume that the variables are binary. Let $\config$ denote a $\exchvars$-configuration,
i.e., a function $\config:\exchvars\rightarrow\{0,1\}$. The set $\{\tau_{\config}^{\exchvars}\ \vert\ \forall\ \text{configuration }\config\}$
is the collection of $\exchvars$-cluster auxiliary variables. Since
$\exchvars$ is exchangeable, all nodes in $\exchvars$ belong to
the same node orbit; we call this node orbit $\vorb(\exchvars)$.
Similarly, $\eorb(\exchvars)$ and $\aorb(\exchvars)$ denote the
single edge and arc orbit that contains all edges and arcs in $\exchvars$
respectively. Let $u_{1},u_{2}$ be two distinct nodes in $\exchvars$.
To enforce consistency between the cluster $\exchvars$ and the edge
$\{u_{1},u_{2}\}$ in the ground model, we introduce the constraints
\begin{equation}
\exists\tau^{\exchvars}:\,\sum_{\config\, s.t.\,\config(u_{i})=s_{i}}\tau_{\config}^{\exchvars}=\tau_{\pair{u_{1}}{s_{1},}\pair{u_{2}}{s_{2}}}\quad\forall s_{i}\in\{0,1\}\label{eq:ground-cluster}
\end{equation}
These constraints correspond to using intersection sets of size two,
which can be shown to be the exact characterization of the marginal
polytope involving variables in $\exchvars$ if the graphical model
only has pairwise potentials. If higher-order potentials are present,
a tighter relaxation could be obtained by using larger intersection
sets together with the techniques described below.

The constraints in (\ref{eq:ground-cluster}) can be methodically
lifted by replacing occurrences of ground variables with lifted variables
at the orbit level. First observe that in place of the grounded variables
$\tau_{\pair{u_{1}}{s_{1},}\pair{u_{2}}{s_{2}}}$, the lifted local
relaxation has three corresponding lifted variables, $\bar{\tau}_{\be(\exchvars):00},\,\bar{\tau}_{\be(\exchvars):11}$
and $\bar{\tau}_{\mathbf{a}(\exchvars):01}$. Second, we consider
the orbits of the set of configurations $\config$. Since $\exchvars$
is exchangeable, there can be only $\left|\exchvars\right|+1$ $\exchvars$-configuration
orbits; each orbit contains all configurations with precisely $k$
1's where $k=0\dots\left|\exchvars\right|$. Thus, instead of the
$2^{\left|\exchvars\right|}$ ground auxiliary variables, we only
need $\left|\exchvars\right|+1$ lifted cluster variables. Further
manipulation leads to the following set of constraints, which we call
\emph{lifted exchangeable polytope} constraints.
\begin{thm}
\label{thm:exch-constraints}Let $\exchvars$ be an exchangeable cluster
of size $n$; $\eorb(\exchvars)$ and $\aorb(\exchvars)$ be the single
edge and arc orbit of the graphical model that contains all edges
and arcs in $\exchvars$ respectively; $\bar{\tau}$ be the lifted
marginals. Then there exist $c_{k}^{\exchvars}\ge0$, $k=0\ldots n$
such that
\begin{eqnarray*}
\sum_{k=0}^{n-2}\frac{(n-k)(n-k-1)}{n(n-1)}c_{k}^{\exchvars} & = & \bar{\tau}_{\be(\exchvars):00}\\
\sum_{k=0}^{n-2}\frac{(k+1)(k+2)}{n(n-1)}c_{k+2}^{\exchvars} & = & \bar{\tau}_{\be(\exchvars):11}\\
\sum_{k=0}^{n-2}\frac{(n-k-1)(k+1)}{n(n-1)}c_{k+1}^{\exchvars} & = & \bar{\tau}_{\aorb(\exchvars):01}
\end{eqnarray*}
\end{thm}
\begin{proof}
See the supplementary material.
\end{proof}
In contrast to the lifted local and cycle relaxations, the number
of variables and constraints in the lifted exchangeable relaxation
depends linearly on the domain size of the first-order model. From
the lifted local constraints given by \cite{bui13uai}, $\bar{\tau}_{\eorb(\exchvars):00}+\bar{\tau}_{\be(\exchvars):11}+2\bar{\tau}_{\aorb(\exchvars):01}=1$.
Substituting in the expression involved $\tilde{c}_{k}^{\exchvars}$,
we get $\sum_{k=0}^{n}c_{k}^{\exchvars}=1$. Intuitively, $c_{k}^{\exchvars}$
represents the approximation of the marginal probability $\Pr(\sum_{i\in\exchvars}x_{i}=k)$
of having precisely $k$ ones in $\exchvars$. 

As proved by \cite{bui12lifted}, groundings of unary predicates in
Markov Logic Networks (MLNs) gives rise to exchangeable clusters.
Thus, for MLNs, the above theorem immediately suggests a tightening
of the relaxation: for every unary predicate of a MLN, add a new set
of constraints as above to the existing lifted local (or cycle) optimization
problem. Although it is not the focus of our paper, we note that this
should also improve the lifted MAP inference results of \cite{bui13uai}.
For example, in the case of a symmetric complete graphical model,
lifted MAP inference using the linear program given by these new constraints
would find the exact $k$ that maximizes $\Pr(x_{\exchvars})$, hence
recover the same solution as counting elimination. Marginal inference
may still be inexact due to the tree-reweighted entropy approximation.
We re-emphasize that the complexity of variational inference with
lifted exchangeable constraints is guaranteed to be polynomial in
the domain size, unlike exact methods based on lifted counting elimination
and variable elimination.

\section{Experiments}

In this section, we provide an empirical evaluation of our lifted
tree reweighted (LTRW) algorithm. As a baseline we use a dampened
version of the lifted belief propagation (LBP-Dampening) algorithm
from \cite{singla08lifted}. Our lifted algorithm has all of the same
advantages of the tree-reweighted approach over belief propagation,
which we will illustrate in the results: (1) a convex objective that
can be convergently solved to optimality, (2) upper bounds on the
partition function, and (3) the ability to easily improve the approximation
by tightening the relaxation. Our evaluation includes four variants
of the LTRW algorithm corresponding to using different outer bounds:
lifted local polytope (LTRW-L), lifted cycle polytope (LTRW-C), lifted
local polytope with exchangeable polytope constraints (LTRW-LE), and
lifted cycle polytope with exchangeable constraints (LTRW-CE). The
conditional gradient optimization of the lifted TRW objective terminates
when the duality gap is less than $10^{-4}$ or when a maximum number
of $1000$ iterations is reached. To solve the LP problem during conditional
gradient, we use Gurobi%
\footnote{http://www.gurobi.com/%
}. 

We evaluate all the algorithms using several first-order probabilistic
models. We assume that no evidence has been observed, which results
in a large amount of symmetry. Even without evidence, performing marginal
inference in first-order probabilistic models can be very useful for
maximum likelihood learning \cite{jaimovich07uai}. Furthermore, the
fact that our lifted tree-reweighted variational approximation provides
an upper bound on the partition function enables us to maximize a
lower bound on the likelihood \cite{ChenLearning}, which we demonstrate
in Sec. \ref{sub:Application-to-Learning}. To find the lifted orbit
partition, we use the renaming group as in \cite{bui13uai} which
exploits the symmetry of the unobserved contants in the model. 

Rather than optimize over the spanning tree polytope, which is computationally
intensive, most TRW implementations use a single fixed choice of edge
appearance probabilities, e.g. an (un)weighted distribution obtained
using the matrix-tree theorem. In these experiments, we initialize
the lifted edge appearance probabilities $\bar{\rho}$ to be the most
uniform per-orbit edge-appearance probabilties by solving the optimization
problem $\inf_{\bar{\rho}\in\stp_{[\shatter^{E}]}}(\bar{\rho}-\frac{|V|-1}{|E|})^{2}$
using conditional gradient. Each direction-finding step of this conditional
gradient solves a lifted MST problem of the form $\sup_{\bar{\ea}^{'}\in\stp_{[\shatter^{E}]}}\left\langle -2(\bar{\ea}-\frac{|V|-1}{|E|}),\bar{\ea}^{'}\right\rangle $
using our lifted Kruskal's algorithm, where $\bar{\ea}$ is the current
solution. After this initialization, we fix the lifted edge appearance
probabilities and do not attempt to optimize them further.

\subsection{Test models}

\begin{figure}
\begin{centering}
\includegraphics[clip,scale=0.18]{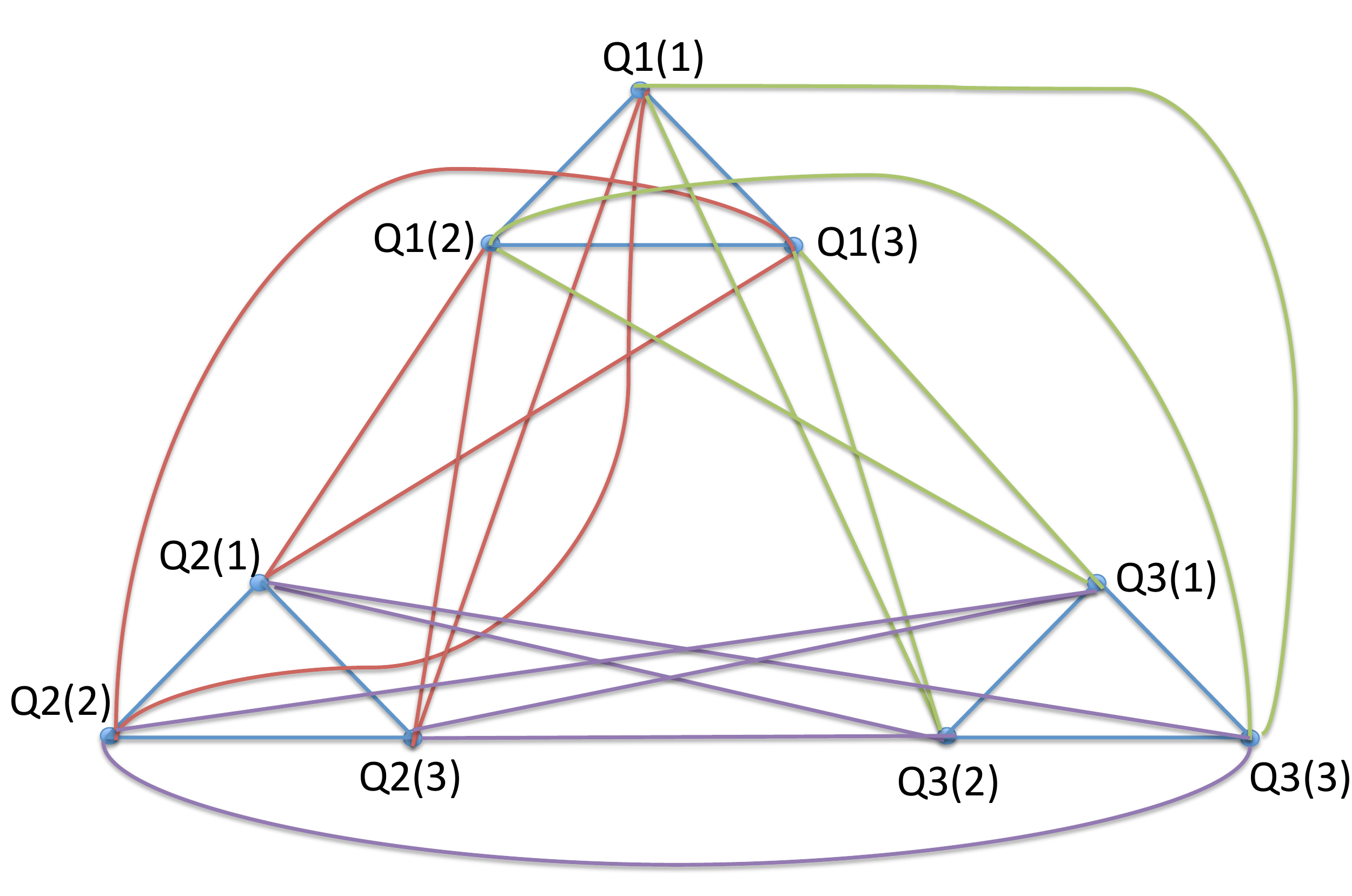}\vspace{-3mm}
\par\end{centering}

\caption{\label{fig:clique-cycle}An example of the ground graphical model
for the Clique-Cycle model (domain size = 3).}
\end{figure}

Fig. \ref{fig:test-models} describes the four test models in MLN
syntax. We focus on the repulsive case, since for attractive models,
all TRW variants and lifted LBP give similar results. The parameter
$W$ denotes the weight that will be varying during the experiments.\emph{
}In all models except \emph{Clique-Cycle}, $W$ acts like the \textquotedblleft{}local
field\textquotedblright{} potential in an Ising model; a negative
(or positive) value of $W$ means the corresponding variable tends
to be in the 0 (or 1) state. \emph{Complete-Graph} is equivalent to
an Ising model on the complete graph of size $n$ (the domain size)
with homogenous parameters. Exact marginals and the log-partition
function can be computed in closed form using lifted counting elimination.
The weight of the interaction clause is set to $-0.1$ (repulsive).
\emph{Friends-Smokers (negated)} is a variant of the Friends-Smokers
model \cite{singla08lifted} where the weight of the final clause
is set to -1.1 (repulsive). We use the method in \cite{bui12lifted}
to compute the exact marginal for the \emph{Cancer} predicate and
the exact value of the log-partition function.\emph{ Lovers-Smokers}
is the same MLN used in \cite{bui13uai} with a full transitive clause
and where we vary the prior of the \emph{Loves} predicate.\emph{ Clique-Cycle}
is a model with 3 cliques and 3 bipartite graphs in between. Its corresponding
ground graphical model is shown in Fig. \ref{fig:clique-cycle}. 

\begin{figure*}
\begin{centering}
\begin{minipage}[t]{0.9\columnwidth}%
\begin{tiny}
Complete Graph \vspace{-2mm}
\begin{eqnarray*}
W & V(x) \\
-0.1 & [x \neq y \wedge (V(x) \Leftrightarrow V(y))] 
\end{eqnarray*}
\end{tiny}%
\end{minipage}%
\begin{minipage}[t]{0.9\columnwidth}%
\begin{tiny}
Friends-Smokers (Negated)\vspace{-2mm}
\begin{eqnarray*}
W & [x \neq y \wedge \neg Friends(x,y)] \\ 
1.4 & \neg Smokes(x) \\ 
2.3 & \neg Cancer(x) \\
1.5 & Smokes(x) \Rightarrow Cancer(x) \\
-1.1 & [x \neq y \wedge Smokes(x) \wedge Friends(x,y) \Rightarrow Smokes(y)]
\end{eqnarray*}
\end{tiny}%
\end{minipage}
\par\end{centering}

\begin{centering}
\begin{minipage}[t]{0.9\columnwidth}%
\begin{tiny}
Lovers-Smokers\vspace{-2mm}
\begin{eqnarray*}
W & [x \neq y \wedge Loves(x,y)] \\
100 & Male(x) \Leftrightarrow !Female(x) \\
2 & Male(x) \wedge Smokes(x) \\
1 & Female(x) \wedge Smokes(x) \\
0.5 & [x \neq y \wedge Male(x) \wedge Female(y) \wedge Loves(x,y)] \\
1 & [x \neq y \wedge Loves(x,y) \wedge (Smokes(x) \Leftrightarrow Smokes(y))] \\
-100 & [x \neq y \wedge y \neq z \wedge z\neq x \wedge Loves(x,y) \wedge Loves(y,z) \wedge Loves(x,z)] 
\end{eqnarray*}
\end{tiny}%
\end{minipage}%
\begin{minipage}[t]{0.9\columnwidth}%
\begin{tiny}
\vspace{2mm}Clique-Cycle\vspace{-2mm}
\begin{eqnarray*} 
W & x \neq y \wedge (Q1(x) \Leftrightarrow \neg Q2(y))\\ 
W & x \neq y \wedge (Q2(x) \Leftrightarrow \neg Q3(y))\\ 
W & x \neq y \wedge (Q3(x) \Leftrightarrow \neg Q1(y)) \\
-W & x \neq y \wedge (Q1(x) \Leftrightarrow Q1(y)) \\ 
-W & x \neq y \wedge (Q2(x)  \Leftrightarrow Q2(y)) \\ 
-W & x \neq y \wedge (Q3(x) \Leftrightarrow Q3(y)) 
\end{eqnarray*}
\end{tiny}%
\end{minipage}
\par\end{centering}

\caption{\label{fig:test-models}Test models}
\end{figure*}

\emph{}

\subsection{Accuracy of Marginals}

Fig. \ref{fig:marginal-accuracy} shows the marginals computed by
all the algorithms as well as exact marginals on the Complete-Graph
and Friends-Smokers models. We do not know how to efficiently perform
exact inference in the remaining two models, and thus do not measure
accuracy for them. The result on complete graphs illustrates the clear
benefit of tightening the relaxation: LTRW-Local and LBP are inaccurate
for moderate $W$, whereas cycle constraints and, especially, exchangeable
constraints drastically improve accuracy. As discussed earlier, for
the case of symmetric complete graphical models, the exchangeable
constraints suffice to exactly characterize the marginal polytope.
As a result, the approximate marginals computed by LTRW-LE and LTRW-CE
are almost the same as the exact marginals; the very small difference
is due to the entropy approximation. On the Friends-Smokers (negated)
model, all LTRW variants give accurate marginals while lifted LBP
even with very strong dampening ($0.9$ weight given to previous iterations'
messages) fails to converge for $W<2$. We observed that LTRW-LE gives
the best trade-off between accuracy and running time for this model.
Note that we do not compare to ground versions of the lifted TRW algorithms
because, by Theorem \ref{thm:lifted-trw-equivalent}, the marginals
and log-partition function are the same for both. 

\begin{figure}
\begin{centering}
\vspace{-2mm}\includegraphics[scale=0.25]{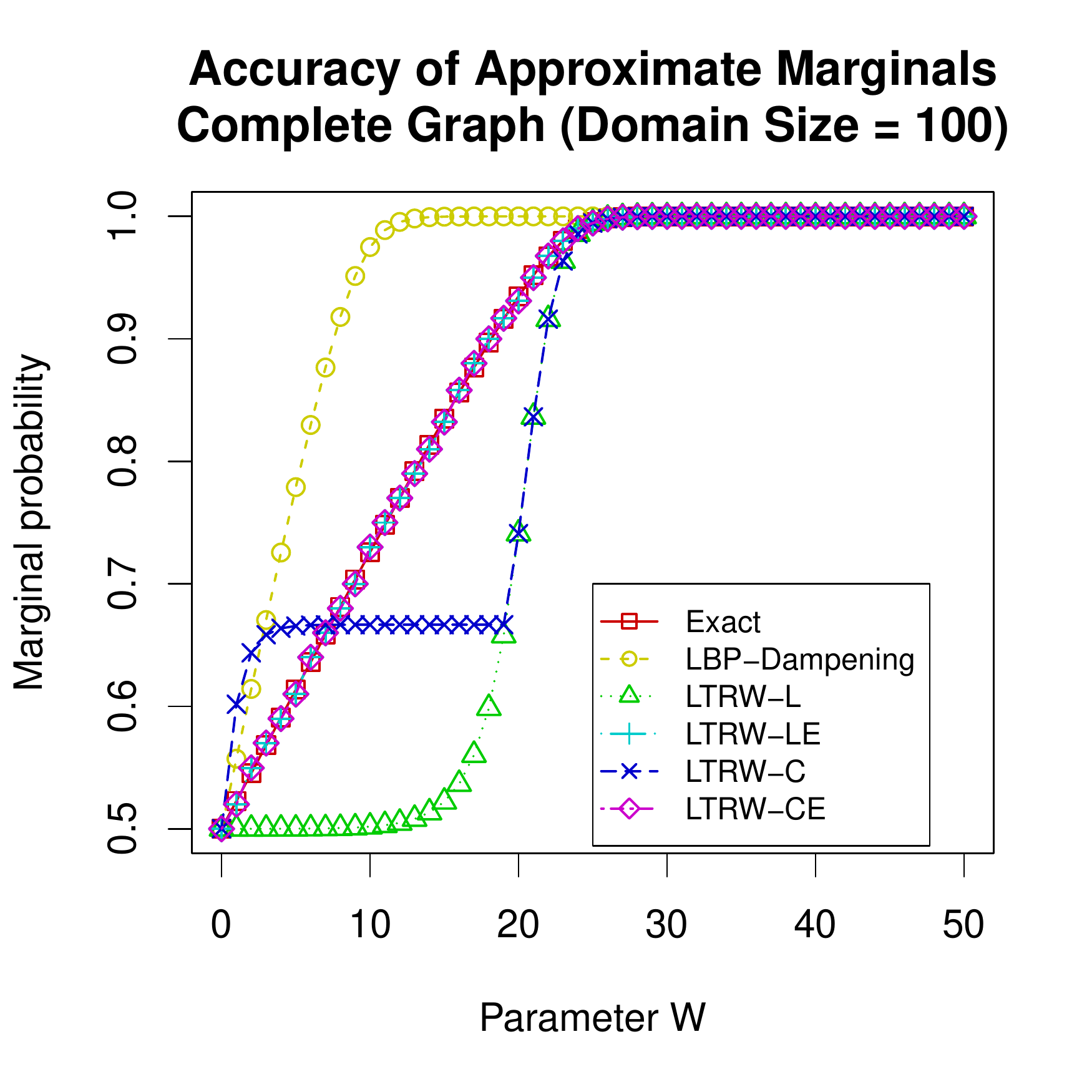}\includegraphics[scale=0.25]{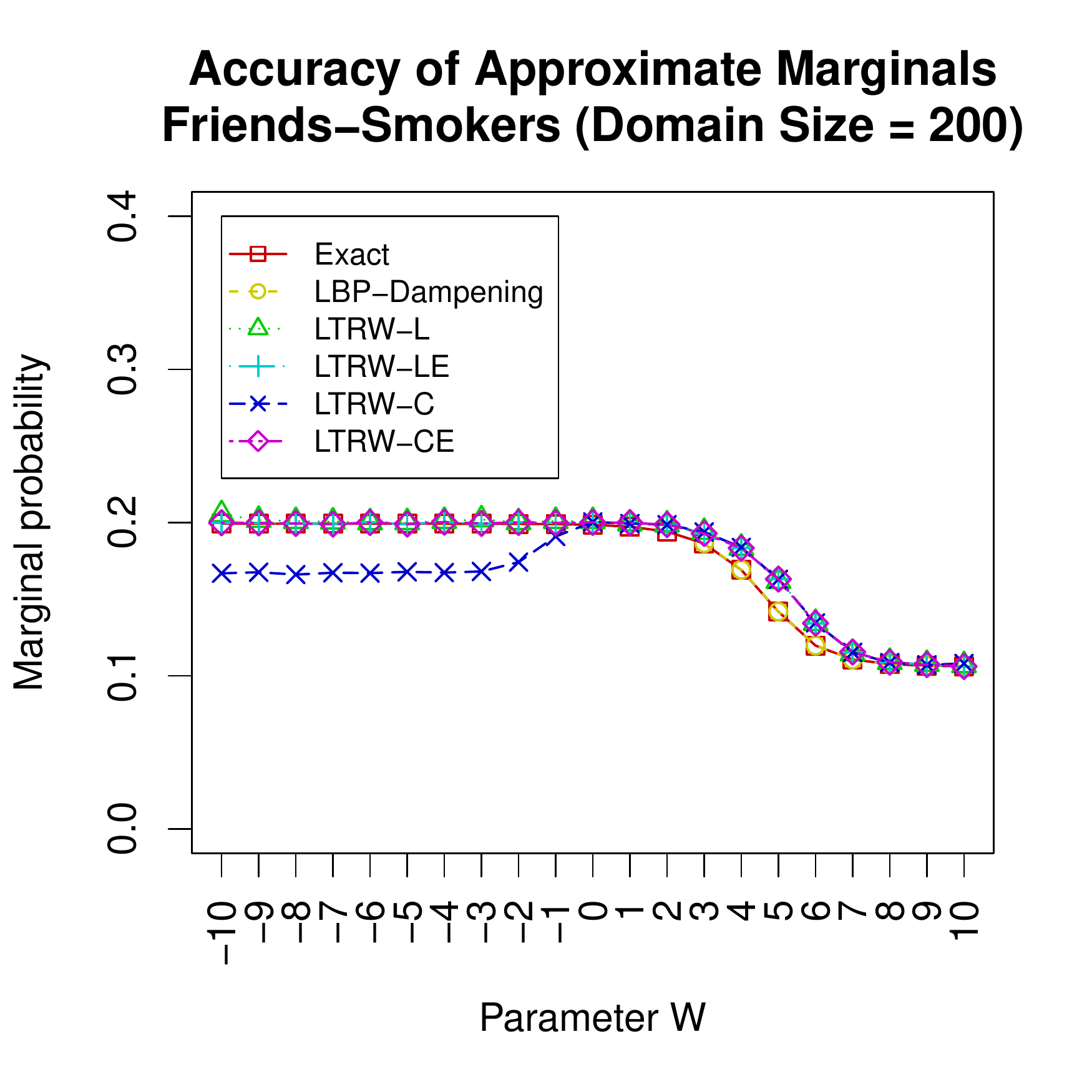}\vspace{-4mm}\caption{\label{fig:marginal-accuracy}Left: marginal accuracy for complete
graph model. Right: marginal accuracy for $Pr(Cancer(x))$ in Friends-Smokers
(neg). Lifted TRW variants using different outer bounds: L=local,
C=cycle, LE=local+exchangeable, CE=cycle+exchangeable (best viewed
in color).}

\par\end{centering}

\end{figure}

\subsection{Quality of Log-Partition Upper bounds}

Fig. \ref{fig:bound-quality} plots the values of the upper bounds
obtained by the LTRW algorithms on the four test models. The results
clearly show the benefits of adding each type of constraint to the
LTRW, with the best upper bound obtained by tightening the lifted
local polytope with both lifted cycle and exchangeable constraints.
For the Complete-Graph and Friends-Smokers model, the log-partition
approximation using exchangeable polytope constraints is very close
to exact. In addition, we illustrate lifted LBP's\emph{ }approximation
of the log-partition function on the Complete-Graph (note it is non-convex
and not an upper bound).\emph{}

\begin{figure*}
\begin{centering}
\includegraphics[scale=0.25]{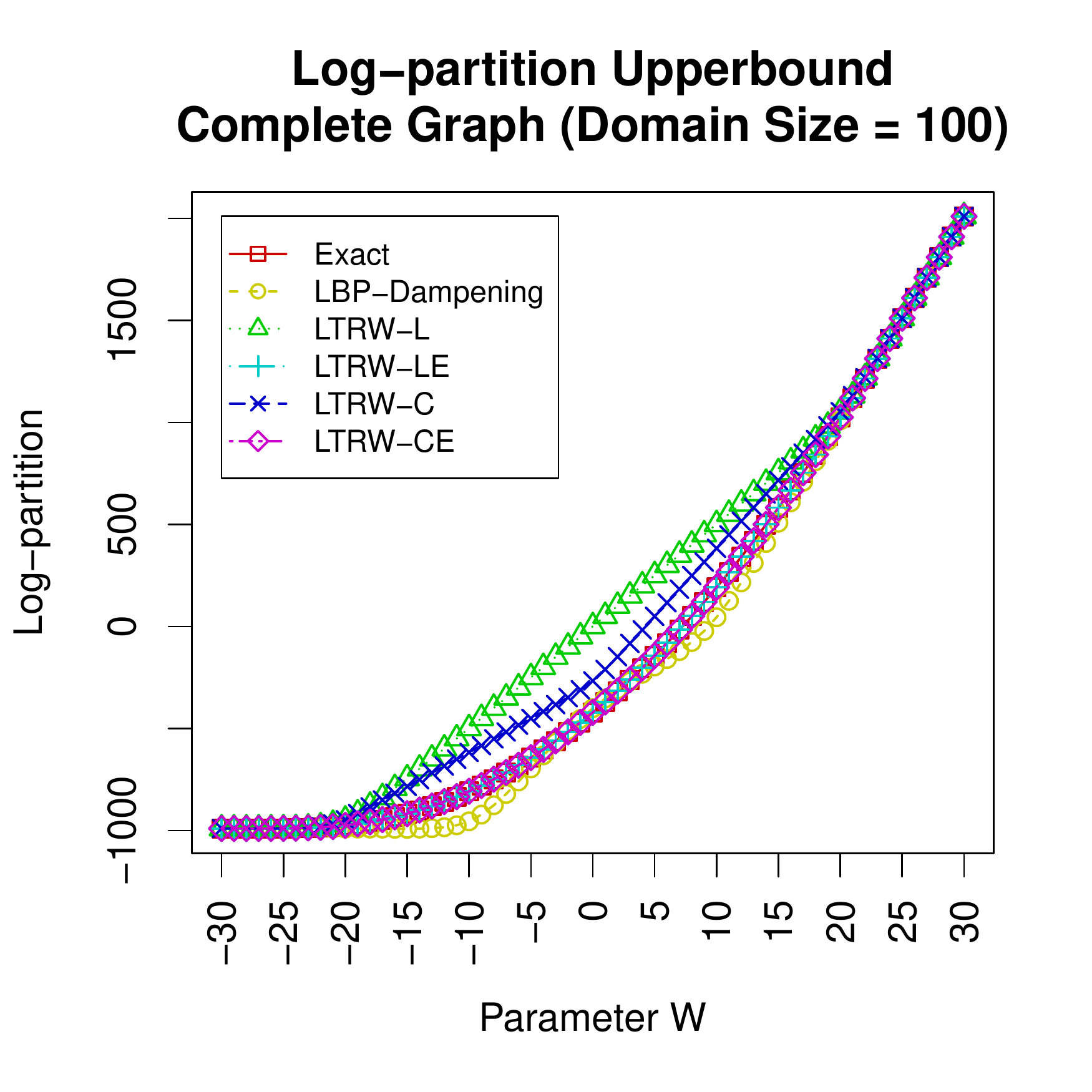}\includegraphics[scale=0.25]{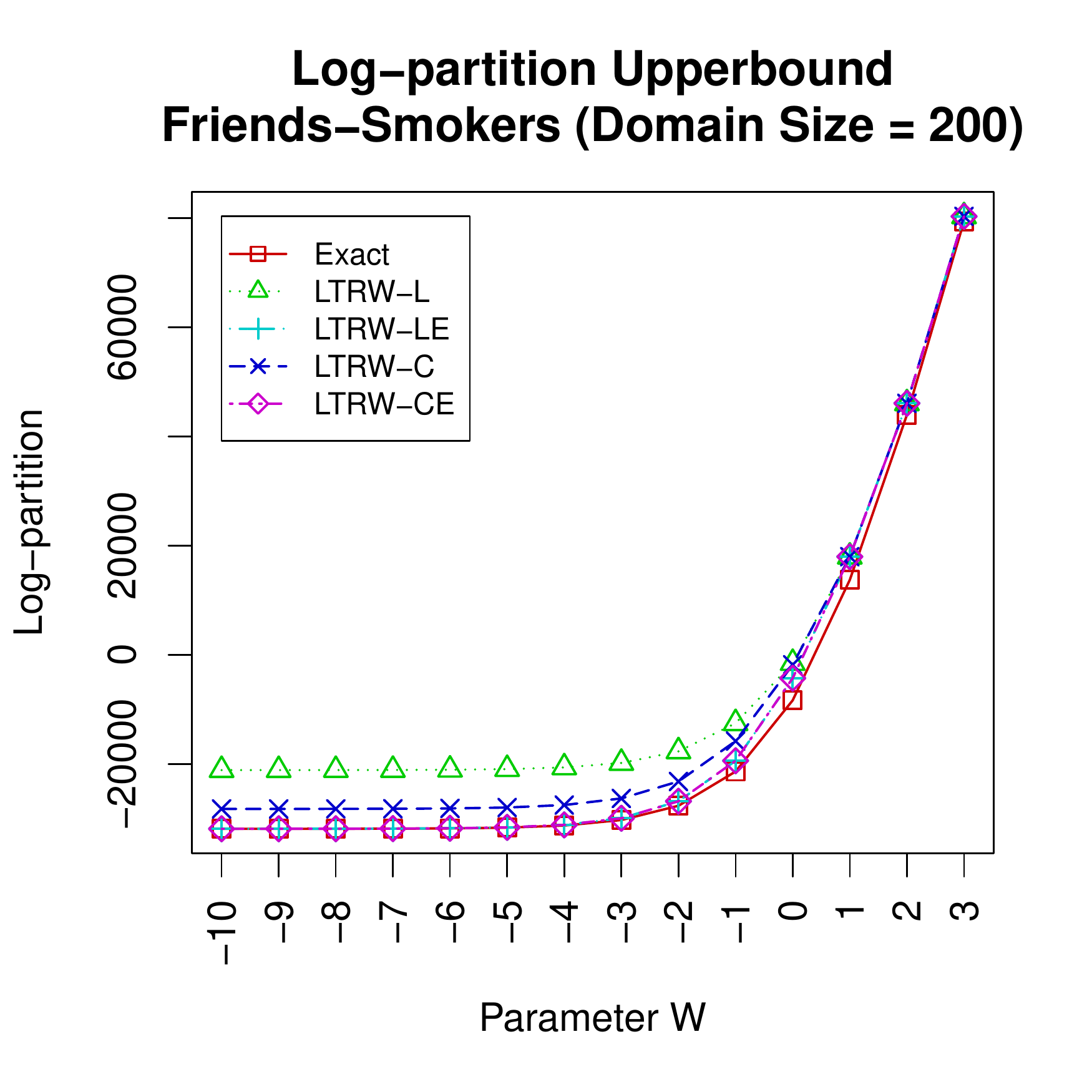}\includegraphics[scale=0.25]{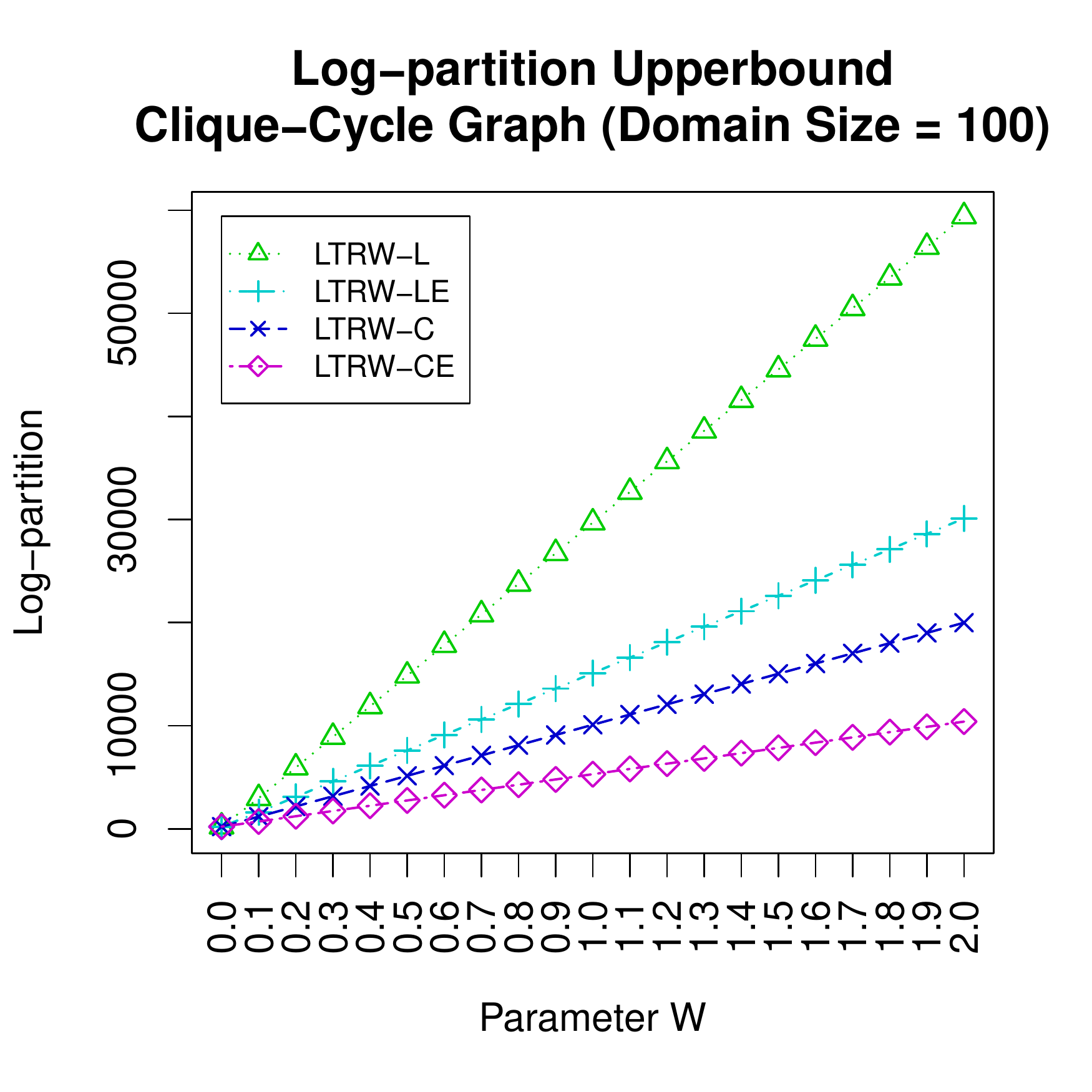}\includegraphics[scale=0.25]{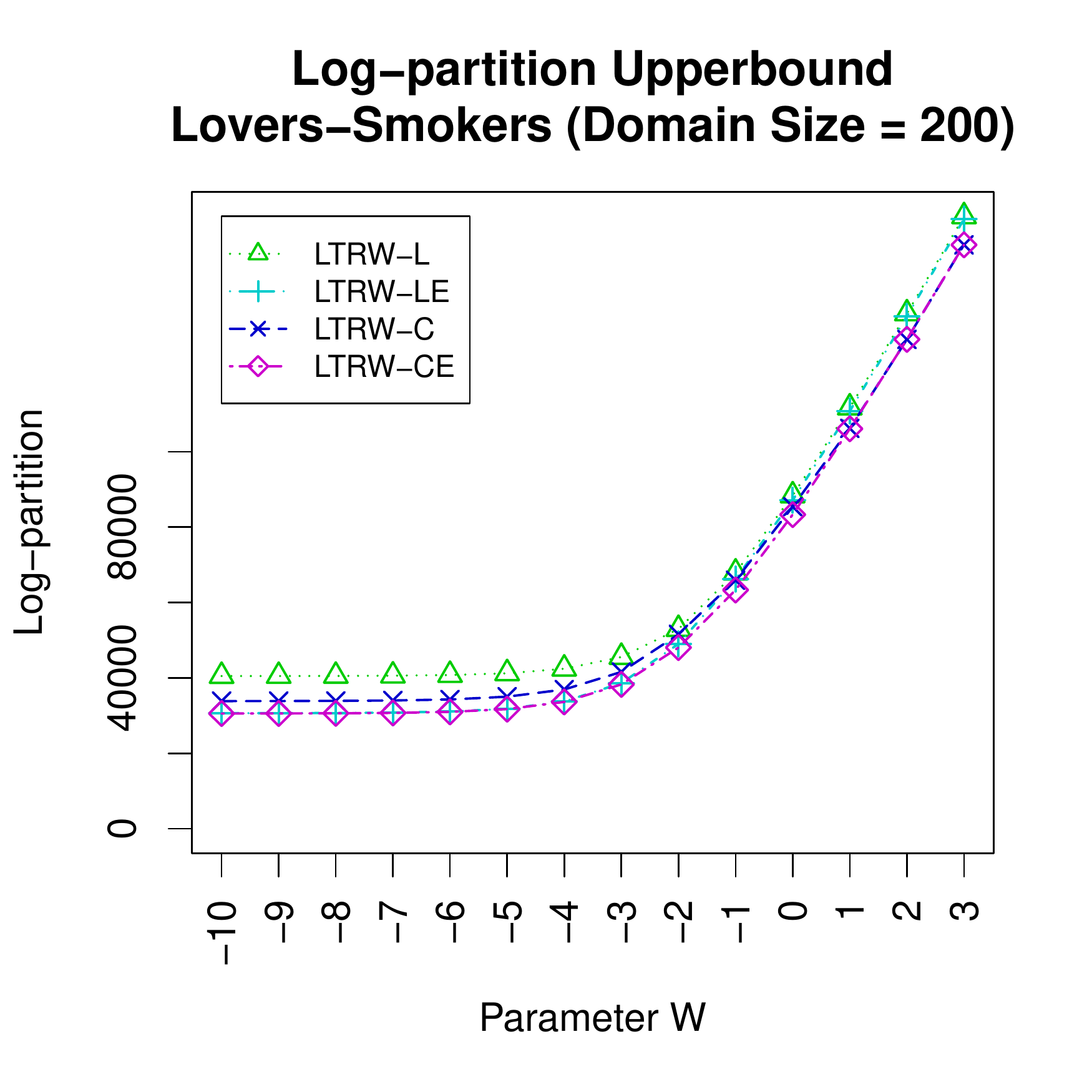}\vspace{-5mm}
\par\end{centering}

\caption{\label{fig:bound-quality}Approximations of the log-partition function
on the four test models from Fig. \ref{fig:test-models} (best viewed
in color).}
\end{figure*}

\subsection{Running time\label{sub:Running-time}}

As shown in Table \ref{tab:runtime}, lifted variants of TRW are order-of-magnitudes
faster than the ground version. Interestingly, lifted TRW with local
constraints is observed to be faster as the domain size increase;
this is probably due to the fact that as the domain size increases,
the distribution becomes more peak, so marginal inference becomes
more similar to MAP inference. Lifted TRW with local and exchangeable
constraints requires a smaller number of conditional gradient iterations,
thus is faster; however note that its running time slowly increases
since the exchangeable constraint set grows linearly with domain size.\emph{}

LBP\textquoteright{}s lack of convergence makes it difficult to have
a meaningful timing comparison with LBP. For example, LBP did not
converge for about half of the values of $W$ in the \emph{Lovers-Smokers}
model, even after using very strong dampening. We did observe that
when LBP converges, it is much faster than LTRW. We hypothesize that
this is due to the message passing nature of LBP, which is based on
a fixed point update whereas our algorithm is based on Frank-Wolfe.

\begin{table}
\begin{centering}
{\small%
\begin{tabular}{|c|c|c|c|c|c|}
\hline 
{\scriptsize Domain size} & {\scriptsize 10} & {\scriptsize 20} & {\scriptsize 30} & {\scriptsize 100} & {\scriptsize 200}\tabularnewline
\hline 
{\scriptsize TRW-L} & {\scriptsize 138370} & {\scriptsize 609502} & {\scriptsize 1525140} & - & -\tabularnewline
\hline 
{\scriptsize LTRW-L} & {\scriptsize 3255} & {\scriptsize 3581} & {\scriptsize 3438} & {\scriptsize 1626} & {\scriptsize 1416}\tabularnewline
\hline 
{\scriptsize LTRW-LE} & {\scriptsize 681} & {\scriptsize 703} & {\scriptsize 721} & {\scriptsize 1033} & {\scriptsize 1307}\tabularnewline
\hline 
\end{tabular}}\vspace{-2mm}
\par\end{centering}

\caption{\label{tab:runtime}Ground vs lifted TRW runtime on Complete-Graph
(milliseconds)}
\end{table}

\subsection{Application to Learning\label{sub:Application-to-Learning}}

We now describe an application of our algorithm to the task of learning
relational Markov networks for inferring protein-protein interactions
from noisy, high-throughput, experimental assays \cite{jaimovich2006towards}.
This is equivalent to learning the parameters of an exponential family
random graph model \cite{robins2007introduction} where the edges
in the random graph represents the protein-protein interactions. Despite
fully observed data, maximum likelihood learning is challenging because
of the intractability of computing the log-partition function and
its gradient. In particular, this relational Markov network has over
330K random variables (one for each possible interaction of 813 variables)
and tertiary potentials. However, Jaimovich et al. \cite{jaimovich07uai}
observed that the partition function in relational Markov networks
is highly symmetric, and use lifted LBP to efficiently perform approximate
learning in running time that is independent of the domain size. They
use their lifted inference algorithm to visualize the (approximate)
likelihood landscape for different values of the parameters, which
among other uses characterizes the robustness of the model to parameter
changes. 

We use precisely the same procedure as \cite{jaimovich07uai}, substituting
lifted BP with our new lifted TRW algorithms. The model has three
parameters: $\theta_1$, used in the single-node potential to specify
the prior probability of a protein-protein interaction; $\theta_{111}$,
part of the tertiary potentials which encourages cliques of three
interacting proteins; and $\theta_{011}$, also part of the tertiary
potentials which encourages chain-like structures where proteins $A,B$
interact, $B,C$ interact, but $A$ and $C$ do not (see supplementary
material for the full model specification as an MLN). We follow their
two-step estimation procedure, first estimating $\theta_1$ in the
absence of the other parameters (the maximum likelihood, BP, and TRW
estimates of this parameter coincide, and estimation can be performed
in closed-form: $\ensuremath{\theta_{1}^{*}}=-5.293$). Next, for
each setting of $\theta_{111}$ and $\theta_{011}$ we estimate the
log-partition function using lifted TRW with the cycle+exchangeable
vs. local constraints only. Since TRW is an upper bound on the log-partition
function, these provide lower bounds on the likelihood.

Our results are shown in Fig. \ref{fig:ppi}, and should be compared
to Fig. 7 of \cite{jaimovich07uai}. The overall shape of the likelihood
landscapes are similar. However, the lifted LBP estimates of the likelihood
have several local optima, which cause gradient-based learning with
lifted LBP to reach different solutions depending on the initial setting
of the parameters. In contrast, since TRW is convex, any gradient-based
procedure would reach the global optima, and thus learning is much
easier. Interestingly, we see that our estimates of the likelihood
have a significantly smaller range over these parameter settings than
that estimated by lifted LBP. Moreover, the high-likelihood parameter
settings extends to larger values of $\theta_{111}$. For all algorithms
there is a sudden decrease in the likelihood at $\theta_{011}>0$
(not shown in the figure).

\begin{figure}
\begin{centering}
\includegraphics[scale=0.23]{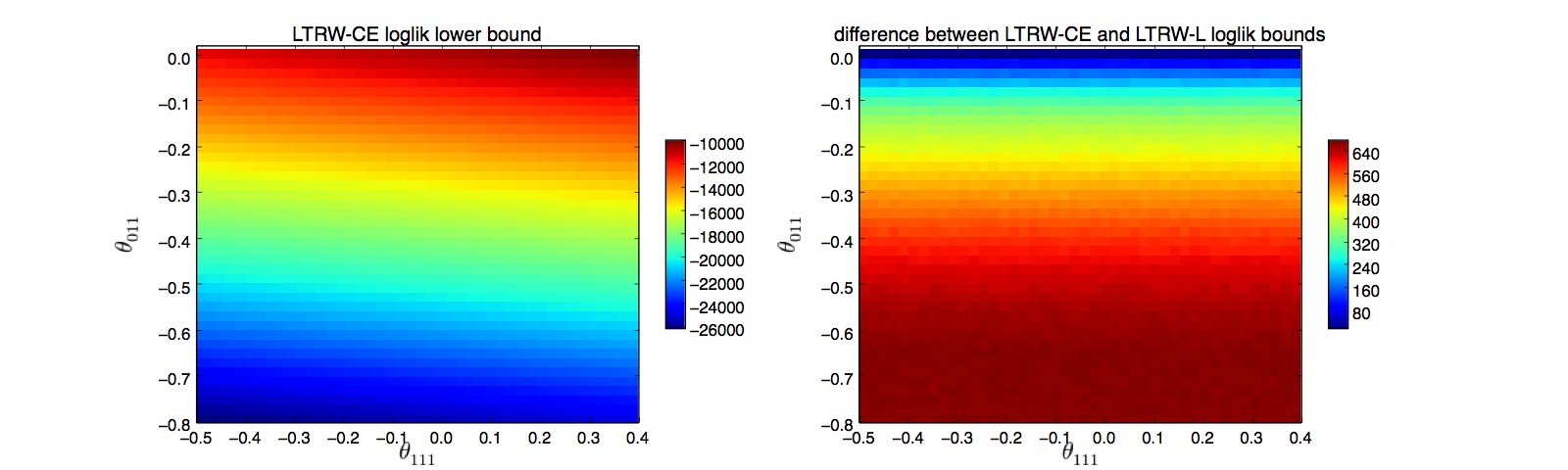}
\par\end{centering}

\begin{centering}
\vspace{-3mm}
\par\end{centering}

\caption{\label{fig:ppi}Log-likelihood lower-bound obtained using lifted TRW
with the cycle and exchangeable constraints (CE) for the same protein-protein
interaction data used in \cite{jaimovich07uai} (left) (c.f. Fig.
7 in \cite{jaimovich07uai}). Improvement in lower-bound after tightening
the local constraints (L) with CE (right).}
\vspace{-2.5mm}
\end{figure}

\section{Discussion and Conclusion}

Lifting partitions used by lifted and counting BP \cite{singla08lifted,kersting09counting}
can be coarser than orbit partitions. In graph-theoretic terms, these
partitions are called \emph{equitable }partitions. If each equitable
partition cell is thought of as a distinct node color, then among
nodes with the same color, their neighbors must have the same color
histogram. It is known that orbit partitions are always equitable,
however the converse is not always true \cite{godsil01agt}.

Since equitable partition can be computed more efficiently and potentially
leads to more compact lifted problems, the following question naturally
arises: can we use equitable partition in lifting the TRW problem?
Unfortunately, a complete answer is non-trivial. We point out here
a theoretical barrier due to the interplay between the spanning tree
polytope and the equitable partition of a graph.

Let $\varepsilon$ be the coarsest equitable partition of edges of
$\gm$. We give an example graph in the supplementary material (see
example \ref{example:equitable-stp}) where the symmetrized spanning
tree polytope corresponding to the equitable partition $\varepsilon$,
$\symsub{\stp}{\epsilon}=\stp(\gm)\cap\Real_{[\varepsilon]}^{|E|}$
is an empty set. When $\symsub{\stp}{\epsilon}$ is empty, the consequence
is that if we want $\ea$ to be within $\mathbb{T}$ so that $\Aapprox(.,\ea)$
is guaranteed to be a convex upper bound of the log-partition function,
we cannot restrict $\ea$ to be consistent with the equitable partition.
In lifted and counting BP, $\ea\equiv1$ so it is clearly consistent
with the equitable partition; however, one loses convexity and upper
bound guarantee as a result. This suggests that there might be a trade-off
between the compactness of the lifting partition and the quality of
the entropy approximation, a topic deserving the attention of future
work.

In summary, we presented a formalization of lifted marginal inference
as a convex optimization problem and showed that it can be efficiently
solved using a Frank-Wolfe algorithm. Compared to previous lifted
variational inference algorithms, in particular lifted belief propagation,
our approach comes with convergence guarantees, upper bounds on the
partition function, and the ability to improve the approximation (e.g.
by introducing additional constraints) at the cost of small additional
running time. 

A limitation of our lifting method is that as the amount of soft evidence
(the number of distinct individual objects) approaches the domain
size, the behavior of lifted inference approaches ground inference.
The wide difference in running time between ground and lifted inference
suggests that significant efficiency can be gained by solving an approximation
of the orignal problem that is more symmetric \cite{van2013complexity,kersting10aaai,singla2010approximate,braz2009anytime}.
One of the most interesting open questions raised by our work is how
to use the variational formulation to perform approxiate lifting.
Since our lifted TRW algorithm provides an upper bound on the partition
function, it is possible that one could use the upper bound to guide
the choice of approximation when deciding how to re-introduce symmetry
into an inference task.

\textbf{Acknowledgements}: Work by DS supported by DARPA PPAML program
under AFRL contract no. FA8750-14-C-0005.

{\small\bibliographystyle{plainnat}
\bibliography{riedel,sontag_main,sontag_paper}
}

\cleardoublepage{}
\section*{Supplementary Materials for ``Lifted Tree-Reweighted Variational
Inference''}

We present (1) proofs not given in the main paper, (2) full pseudo-code
for the lifted Kruskal's algorithm for finding maximum spanning tree
in symmetric graphs, and (3) additional details of the protein-protein
interaction model.

\textbf{Proof of Theorem \ref{thm:equivalent}}.
\begin{proof}
The lifting group $\Aut_{\parti}$ stabilizes both the objective function
and the constraints of the convex optimization problem in the LHS
of Eq. (\ref{eq:equivalent}). The equality is then established using
Lemma 1 in \cite{bui13uai}.
\end{proof}
We state and prove a lemma about the symmetry of the bounds $B$ and
$B^{*}$ that will be used in subsequent proofs.
\begin{lem}
\label{lem:symmetry-of-B}let $\naut$ be an automorphism of the graphical
model $\gm$, then $\negEntApprox(\tau,\ea)=\negEntApprox(\tau^{\pi},\ea^{\pi})$
and $B(\theta,\ea)=B(\theta^{\pi},\ea^{\pi})$.\end{lem}
\begin{proof}
The intuition is that since the entropy bound $\negEntApprox$ is
defined on the graph structure of the graphical model $\gm$, it inherits
the symmetry of $\gm$. This can be verified by viewing the graph
automorphism$\pi$ as a bijection from nodes to nodes and edges to
edges, and so $\negEntApprox(\tau^{\pi},\ea^{\pi})$ simply rearranges
the summation inside $\negEntApprox(\tau,\ea)$. 
\begin{eqnarray*}
B^{*}(\tau,\ea) & = & -\sum_{v\in V(\gm)}H(\tau_{v})+\sum_{e\in E(\gm)}I(\tau_{e})\ea_{e}\\
 & = & -\sum_{v\in V(\gm)}H(\tau_{\naut(v)})+\sum_{e\in E(\gm)}I(\tau_{\naut(e)})\ea_{\naut(e)}\\
 & = & B^{*}(\tau^{\naut},\ea^{\naut})
\end{eqnarray*}

We now show same symmetry applies to the log-partition upper bound
$\Aapprox$. Let $\outbound(\gm)$ be an outer bound of the marginal
polytope $\Mean(\gm)$ such that $\Mean(\gm)\subset\outbound(\gm)\subset\Local(\gm)$.
Note that $\naut$ acts on and stabilizes $\outbound$, i.e., $\outbound^{\naut}=\outbound$.
Thus
\begin{eqnarray*}
B(\para,\ea) & = & \sup_{\tau\in\outbound}\left\langle \para,\tau\right\rangle -B^{*}(\tau,\ea)\\
 & = & \sup_{\tau^{\naut}\in\outbound}\left\langle \para,\tau\right\rangle -B^{*}(\tau,\ea)\\
 & = & \sup_{\tau\in\outbound}\left\langle \para,\tau^{\naut^{-1}}\right\rangle -B^{*}(\tau^{\naut^{-1}},\ea)\\
 & = & \sup_{\tau\in\outbound}\left\langle \para^{\naut},\tau\right\rangle -B^{*}(\tau,\ea^{\naut})\\
 & = & B(\para^{\naut},\ea^{\naut})
\end{eqnarray*}

\end{proof}
\textbf{Proof of Theorem \ref{thm:lifted-trw-equivalent}.}
\begin{proof}
We will show that the condition of theorem \ref{thm:equivalent} holds.
Let us fix a $\ea\in\symsub{\stp}{\varphi^{E}}$. Then $\ea^{\naut}=\ea$
for all $(\naut,\paut)\in\Aut_{\parti}.$ Thus $\negEntApprox(\tau^{\naut},\ea)=\negEntApprox(\tau^{\naut},\ea^{\naut})$.
On the other hand, by Lemma \ref{lem:symmetry-of-B}, $\negEntApprox(\tau^{\naut},\ea^{\naut})=\negEntApprox(\tau,\rho)$.
Thus $\negEntApprox(\tau^{\naut},\ea)=\negEntApprox(\tau,\rho)$.
Note that in case of the overcomplete representation, the action of
the group $\Aut_{\parti}$ is the permuting action of $\naut$; thus,
the TRW bound $\negEntApprox(\tau,\ea)$ (for fixed $\ea\in\symsub{\stp}{\varphi^{E}}$)
is stabilized by the lifting group $\Aut_{\parti}$.
\end{proof}
\textbf{Conditional Gradient (Frank-Wolfe) Algorithm for Lifted TRW}

The pseudo-code is given in Algorithm \ref{alg:ltrw}. Step 2 essentially
solves a lifted MAP problem which we used the same algorithms presented
in \cite{bui13uai} with Gurobi as the main linear programming engine.
Step 3 solves a 1-D constrained convex problem via the golden search
algorithm to find the optimal step size.

\begin{algorithm*}
\begin{algorithmic}[1] 	
\State $k=0$; $\bar{\tau}^{(0)} \gets$ uniform
\State Direction finding via lifted MAP
\[s_k = \arg\max_{s_k\in\overline{\mathcal{O}}} \langle s_k, \bar{\para}-\nabla_{\bar{\tau}} \mathrm{\overline{B^{*}}}(\bar{\tau}_k,\bar{\ea})\rangle \]
\State Step size finding via golden section search
\[\lambda_k = \arg\max_{\lambda\in[0,1]} \lambda \langle s_k-\bar{\tau}_k,\bar{\para}\rangle - \overline{B^{*}}(\bar{\tau}_k(1-\lambda)+s_k\lambda)\]
\State Update $\bar{\tau}_{k+1} =\bar{\tau}_k(1-\lambda_k)+s_k\lambda_k$
\State $k\gets k+1$
\State if not converged go to 2
\end{algorithmic} 

\caption{\label{alg:ltrw}Conditional gradient for optimizing lifted TRW problem}

\end{algorithm*}

\textbf{Proof of Lemma \ref{lem:lifted-mst}}.
\begin{proof}
After considering all edges in $\mathbf{e}_{1}\cup\ldots\cup\eorb_{i}$,
Kruskal's algorithm must form a \emph{spanning} forest of $\mathcal{G}_{i}$
(the forest is spanning since if there remains an edge that can be
used without forming a cycle, Kruskal's algorithm must have used it
already). Since the forest is spanning, the number of edges used by
Kruskal's algorithm at this point is precisely $|V(\mathcal{G}_{i})|-|C(\mathcal{G}_{i})|$.
Similarly, just after considering all edges in $\mathbf{e}_{1}\cup\ldots\cup\eorb_{i-1}$,
the number of edges used is $|V(\mathcal{G}_{i-1})|-|C(\mathcal{G}_{i-1})|$.
Therefore, the number of $\mathbf{\eorb}_{i}$-edges used must be
\[
|V(\mathcal{G}_{i})|-|C(\mathcal{G}_{i})|-[|V(\mathcal{G}_{i-1})|-|C(\mathcal{G}_{i-1})|]=\delta_{V}^{(i)}-\delta_{C}^{(i)}
\]
which is the difference between the number of new nodes (which must
be non-negative) and the number of new connected components (which
could be negative) induced by considering edges in $\eorb_{i}$. Any
MST solution $\ea$ can be turned into a solution $\bar{\rho}$ of
(\ref{eq:spanning-tree-lp-sym}) by letting $\bar{\rho}(\eorb)=\frac{1}{|\eorb|}\sum_{e'\in\eorb}\rho(e')$
. Thus, we obtain a solution $\bar{\ea}(\eorb_{i})=\frac{\delta_{V}^{(i)}-\delta_{C}^{(i)}}{|\eorb_{i}|}$.
\end{proof}
\textbf{Proof of Lemma \ref{lem:counting-connected-component}}. 

We first need an intermediate result.

\global\long\def\bv{\mathbf{v}}

\global\long\def\bz{\mathbf{z}}

\begin{lem}
If $\mathbf{u}$ and $\mathbf{v}$ are two distinct node orbits, and
$ $$\mathbf{u}$ and $\mathbf{v}$ are reachable in the lifted graph
$\bar{\mathcal{G}}$, then for any $u\in\mathbf{u}$, there is some
$v\in\mathbf{v}$ such that $v$ is reachable from $u$.\end{lem}
\begin{proof}
Induction on the length of the $\mathbf{u}$-$\mathbf{v}$ path. Base
case: if $\bu$ and $ $$\bv$ are adjacent, there exists an edge
orbit $\eorb$ incident to both $\bu$ and $\bv$. Therefore, the
exists a ground edge $\{u_{0},v_{0}\}$ in $\eorb$ such that $u_{0}\in\bu$
and $v_{0}\in\bv$. The automorphism mapping $u_{0}\mapsto u$ will
map $v_{0}\mapsto v$ for some node $v$. Clearly $\{u,v\}$ is an
edge, and $v\in\bv$. Main case: assume the statement is true for
all pair of orbits with path length $\le n$. Suppose $\bu-\bv$ is
a path of length $n+1$. Take the orbit $\bz$ right in front of $\bv$
in the path, so that $\bu-\bz$ is a path of length $n$, and $\bz$
and $\bv$ are adjacent. By the inductive assumption, there exists
$z\in\bz$ such that $u$ is connected to $z$. Applying the same
argument of the base case, there exists $v\in\bv$ such that $\{z,v\}$
is an edge. Thus $u$ is connected to $v$.
\end{proof}
We now return to the main proof of lemma \ref{lem:counting-connected-component}.
\begin{proof}
If the ground graph $\mathcal{G}$ has only one component then this
is trivially true. Let $\mathcal{G}_{1}$ and $\mathcal{G}_{2}$ be
two distinct connected components of $\mathcal{G}$, let $u_{1}$
be a node in $\mathcal{G}_{1}$ and $\mathbf{u}$ be the orbit containing
$u_{1}$. Let $v_{2}$ be any node in $\mathcal{G}_{2}$. Since the
lifted graph $\bar{\gm}$ is connected, all orbits are reachable from
one another in $\bar{\gm}$. By the above lemma, there must be some
node $u_{2}\in\mathbf{u}$ reachable from $v_{2}$, hence $u_{2}\in\gm_{2}$
(if $v_{2}\in\mathbf{u}$ then we just take $u_{2}=v_{2}$). This
establishes that the node orbit $\bu$ intersects with both $\gm_{1}$
and $\gm_{2}$. Note that $u_{1}\neq u_{2}$ since otherwise $ $
$\mathcal{G}_{1}=\mathcal{G}_{2}$. Let $\pi$ be the automorphism
that takes $u_{1}$ to $u_{2}$. 

We now show that $\pi(\mathcal{G}_{1})=\mathcal{G}_{2}$. Since $\naut$
maps edges to edges and non-edges to non-edges, it is sufficient to
show that $\naut(V(\gm_{1}))=\naut(V(\gm_{2}))$. Let $z_{1}$ be
a node of $\gm_{1}$ and $z_{2}=\naut(z_{1})$. Since $\gm_{1}$ is
connected, there exists a path from $u_{1}$ to $z_{1}$. But $\naut$
must map this path to a path from $u_{2}$ to $z_{2}$, hence $z_{2}\in V(\gm_{2})$.
Thus $\naut(V(\gm_{1}))\subset V(\gm_{2})$. Now, let $z_{2}$ be
a node of $\gm_{2}$ and let $z_{1}=\naut^{-1}(z_{2})$. By a similar
argument, $\naut^{-1}$ must map the path from $u_{2}$ to $z_{2}$
to a path from $u_{1}$ to $z_{1}$, hence $z_{1}\in V(\gm_{1})$.
Thus $\naut(V(\gm_{2}))\subset V(\gm_{1})$, and we have indeed shown
that $\naut(V(\gm_{1}))=\naut(V(\gm_{2}))$.

Hence all connected components of $\gm$ are isomorphic. Given one
connected component $\gm^{'}$, the number of connected components
of $\gm$ is $|C(\mathcal{G})|=|V(\mathcal{G})|/|V(\mathcal{G}^{'})|$.
\end{proof}
\textbf{Lifted Kruskal's Algorithm}

See algorithm \ref{alg:lifted-kruskal}. The algorithm keeps track
of the list of connected components of the lifted graph $\bar{\gm}_{i}$
in a disjoint-set data structure similar to Kruskal's algorithm. Line
17 and line 20 follow from lemma \ref{lem:counting-connected-component}
and lemma \ref{lem:lifted-mst} respectively. The number of ground
nodes of a lifted graph is computed as the sum of the size of each
of its node orbit.

\begin{algorithm*}[t]
\caption{Lifted Kruskal's algorithm\\
	Find:~$\bar\ea$, a solution of (\ref{eq:spanning-tree-lp-sym}) at orbit level\\
	Input:~lifted graph $\bar\gm$ and its set of edge orbits
}
\label{alg:lifted-kruskal}
\begin{algorithmic}[1]
	\State sort edge orbits in decreasing weight order $(\eorb_1\ldots \eorb_k)$
	\State $C=\emptyset$ \Comment{set of connected comp. of lifted graph $\bar{\gm}_i$ as disjoint set}
	\State $GC=$ empty map \Comment{hashmap from elements of $C$ to their number of ground components}
	\State $numGC=0$ \Comment{number of ground connected comp of $\bar{\gm}_i$}
	\State $numGV=0$ \Comment{number of ground nodes of $\bar{\gm}_i$}
	\State $GV_{max}=\#GroundNode(\bar{\gm})$ \Comment{number of ground nodes of $\bar{\gm}$}
	\For{i=1,\dots,k}
		\State $C_{old} =$ find elements in $C$ containing the end nodes of $\eorb_i$
		\State $GC_{old} = \sum_{S\in C_{old}}GC(S)$
		\State $\delta_V = \sum_{{\bf v}:{\bf v}\in Nb({\bf e}_{i}),\ {\bf v}\not\in C_{old}}|{\bf v}|$
		\State $numGV=numGV+\delta_{V}$
		\State $H =$  union of members of $C_{old}$ and ${\bf e}_{i}$
		\State $C \gets$ remove members of $C_{old}$ and add $H$
		\State ${\bf u}=$ pick a node of $H$; $u=representative({\bf u})$
		\State $H_{fixed} = $ lifted graph of $H$ after fixing the node $u$
		\State $K=$ component contains $\{u\}$ in $H_{fixed}$
		\State $GC(H)\gets\frac{\#GroundNode(H)}{\#GroundNode(K)}$
		\State $\delta_{C}=GC(H)-GC_{old}$
		\State $numGC\gets numGC+\delta_{C}$
		\State $\bar{\ea}(\eorb_{i})=\frac{1}{|\eorb_{i}|}[\delta_{V}-\delta_{C}]$
		\If{ $numGV=GV_{max}$ and $numGC=1$}\Comment{no new ground nodes, 1 ground connected component}
			\State break \Comment{future $\delta_V$ and $\delta_C$ must be $0$}
		\EndIf
	\EndFor
	\State  $\bar{\ea}(\eorb_{j}) = 0$ for all $j=i+1\ldots k$
	\State \textbf{return} $\bar{\ea}$
\end{algorithmic}
\end{algorithm*}

\begin{figure*}
\includegraphics[scale=0.4]{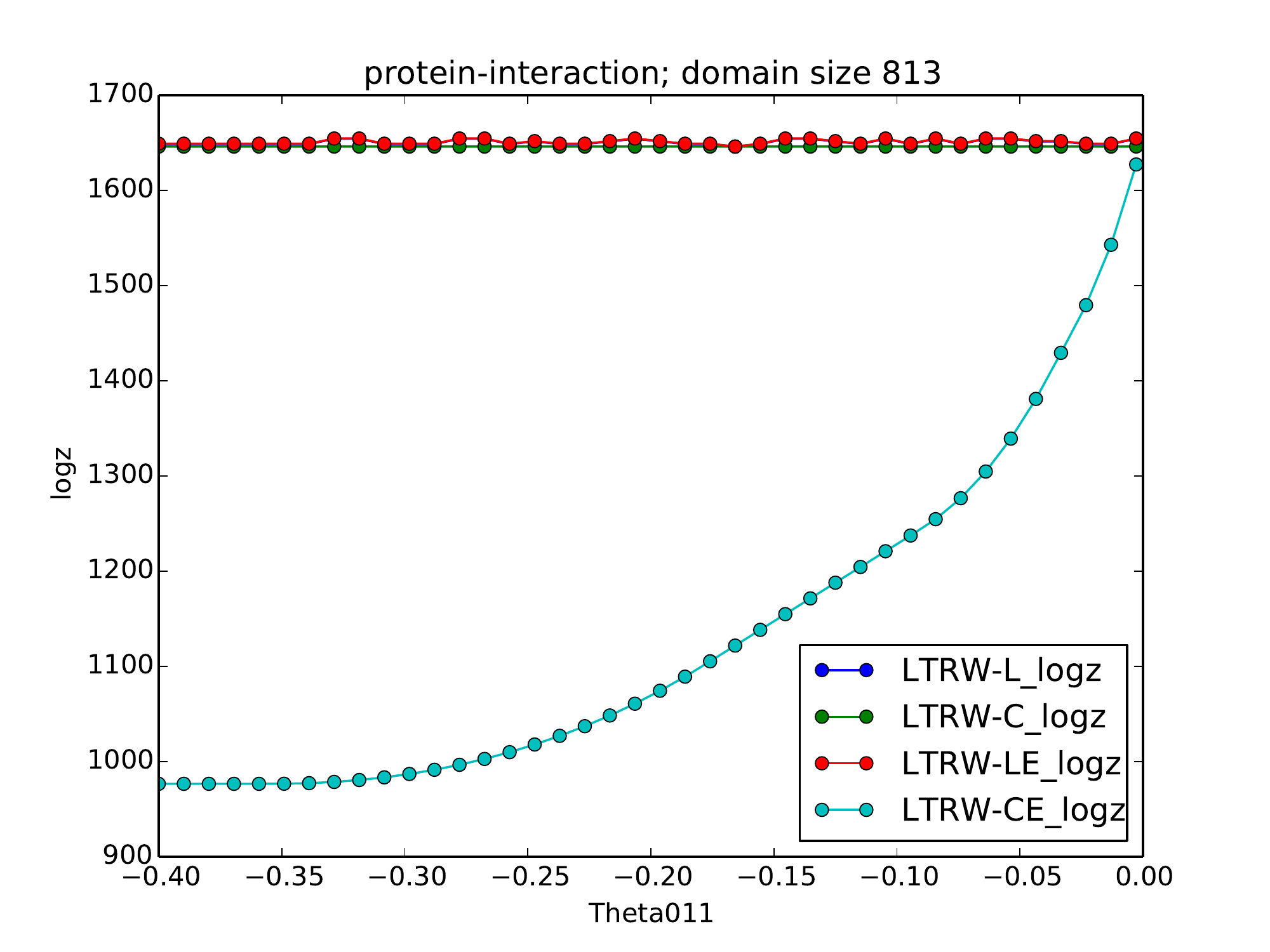}\includegraphics[scale=0.4]{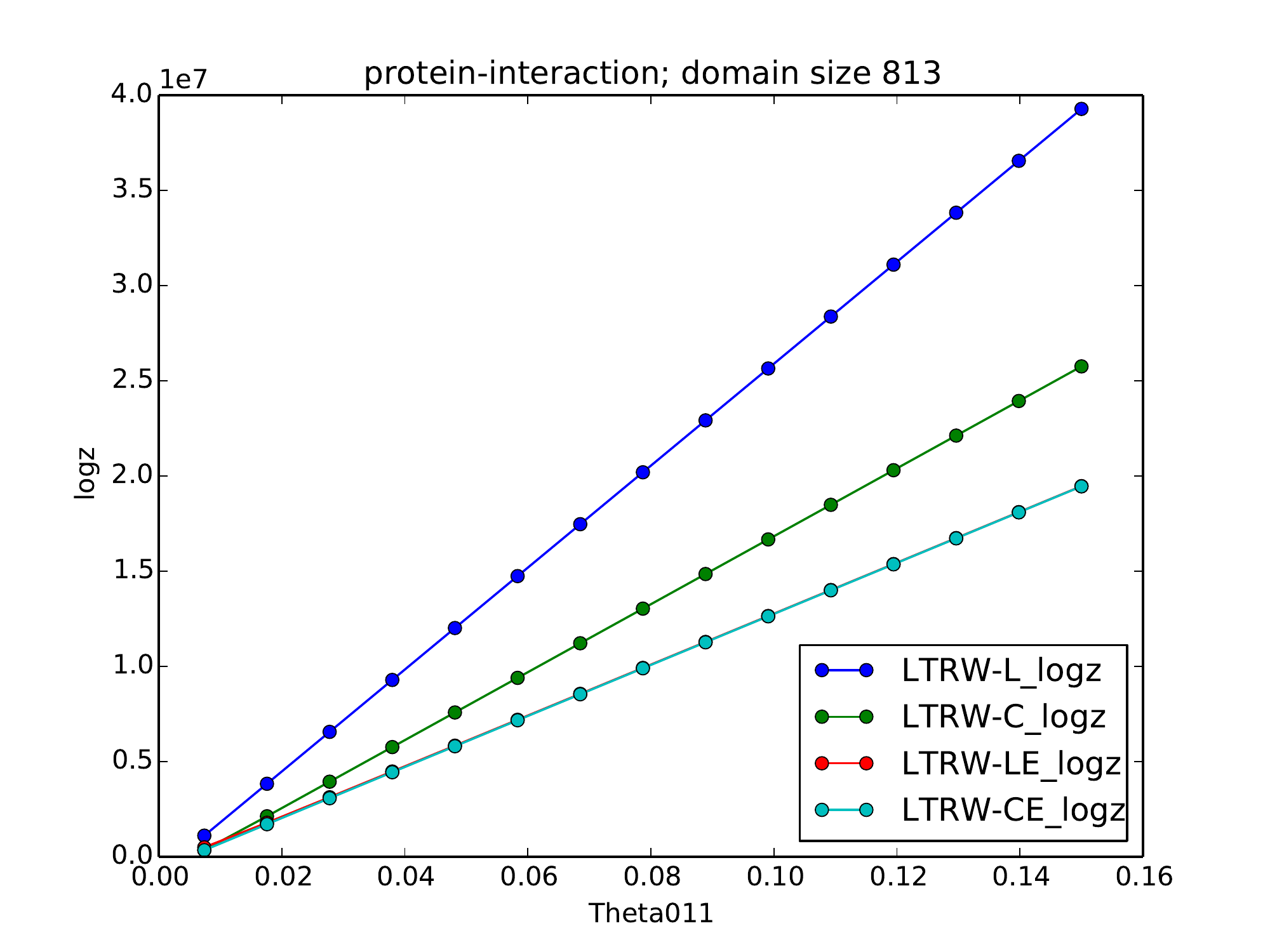}

\caption{\label{fig:ppi-logz}Log-partition bounds for PPI model}
\end{figure*}

\textbf{Proof of Theorem \ref{thm:exch-constraints}}. 

We give two different proofs of this theorem. The first proof demonstrates
how the constraints are derived methodically by lifting the ground
constraints of the exact marginal polytope. The second proof is more
intuitive and illustrates what each variable in the system of constraints
represents conceptually. 

Proof 1.
\begin{proof}
First consider the configuration $\config$ such that $\config(u_{1},u_{2})=(0,0).$
From Eq. (\ref{eq:ground-cluster}), after substituting the ground
variables $\tau_{\config}^{\exchvars}$ by the lifted variables $\bar{\tau}_{|\config|}^{\exchvars}$
where $|\config|$ denotes the number of $1$'s in the configuration,
we obtain the lifted constraint
\[
\exists\bar{\tau}_{0}^{\exchvars}\ldots\bar{\tau}_{n}^{\exchvars}:\,\sum_{\config\ s.t.\ \config(u_{1},u_{2})=(0,0)}\bar{\tau}_{|\config|}^{\exchvars}=\bar{\tau}_{\mathbf{e}(\exchvars):00}
\]
Let us now simplify the summation. Since $\config(u_{1},u_{2})=(0,0)$,
$|\config|$ can range from 0 to $n-2$. For each value of $|\config|=k$
there are $\left(\stackrel{n-2}{k}\right)$ different configurations.
As a result, we can compactly write the above lifted constraint as
\[
\sum_{k=0}^{n-2}\left(\stackrel{n-2}{k}\right)\bar{\tau}_{k}^{\exchvars}=\bar{\tau}_{\be(\exchvars):00}
\]
Note that every edge $\{u_{1},u_{2}\}$ results in exactly the same
constraint. 

Similarly, when $\config(u_{1},u_{2})=(1,1)$ we obtain the lifted
constraint
\[
\sum_{k=0}^{n-2}\left(\stackrel{n-2}{k}\right)\bar{\tau}_{k+2}^{\exchvars}=\bar{\tau}_{\be(\exchvars):11}
\]
and when $\config(u_{1},u_{2})=(0,1)$ we obtain 
\[
\sum_{k=0}^{n-2}\left(\stackrel{n-2}{k}\right)\bar{\tau}_{k+1}^{\exchvars}=\bar{\tau}_{\mathbf{a}(\exchvars):01}
\]
There's no need to consider the $(1,0)$ case since we can always
re-order the pair $(u_{1},u_{2})$.

Finally, let ${c_{k}^{\exchvars}= \choose nk\bar{\tau}_{k}^{\exchvars}}$,
we arrive at the set of constraints of theorem \ref{thm:exch-constraints}. 
\end{proof}
Proof 2.
\begin{proof}
Let $C(0,0)$ denote the number of $(0,0)$ edges in $\exchvars$.
An $(0,0)$ edge is an edge where the two end-nodes receive the assignment
$0$. To show the first equality holds, we use two different ways
to compute the expectation of $C(0,0)$. 

First, $\Expt\ [C(0,0)]$ is the sum of the probability that an edge
is $(0,0)$, summing over all the edges. Due to exchangeability, all
probabilities are the same and is equal to $\bar{\tau}_{\pair{\eorb(\exchvars)}{00}}$,
so 
\[
\Expt[C(0,0)]=\frac{n(n-1)}{2}\bar{\tau}_{\pair{\eorb(\exchvars)}{00}}
\]
Second, $C(0,0)$ conditioned on the event that there are precisely
$k$ $1$'s in $\exchvars$ is $\frac{(n-k)(n-k-1)}{2}$ for $k\le n-2$
and $0$ if $k>n-2$. Now let $c_{k}^{\exchvars}$ be the probability
of this event, then 
\begin{eqnarray*}
\Expt[C(0,0)] & = & \Expt[\Expt[C(0,0)\vert\sum_{i\in\exchvars}x_{i}=k]]\\
 & = & \sum_{k=0}^{n-2}\frac{(n-k)(n-k-1)}{2}c_{k}^{\exchvars}
\end{eqnarray*}

This shows the first equality holds. The second equality is obtained
similarly by considering the expectation of $C(1,1)$, the number
of $(1,1)$ edges in $\exchvars$. The last equality is obtained by
considering the expectation of $C(0,1)$, the number of $(0,1)$ arcs.
\end{proof}
\textbf{Protein-Protein Interaction (PPI) Model}

The PPI model we considered is exactly the same as an exponential
family random graph model with 3 graph statistics: edge count, 2-chain
(triangle with a missing edge) count, and triangle count. The model
specification in MLN (with an additional clause enforcing undirectedness
of the random graph) is 
\begin{eqnarray*}
\frac{1}{2}\theta_{11} &  & r(x,y)\\
\frac{1}{2}\theta_{110} &  & r(x,y)\wedge r(x,z)\wedge!r(y,z)\\
\frac{1}{6}\theta_{111} &  & r(x,y)\wedge r(y,z)\wedge r(x,z)\\
-\infty &  & !(r(x,y)\leftrightarrow r(y,x))
\end{eqnarray*}

For this model, the edge appearance $\bar{\ea}$ is initialized so
that the edges correspond to the hard clause are always selected.
The set $\{r(x,y)\ |\ x\ \text{fixed,\ }y\ \text{varies}\}$ is an
exchangeable cluster, thus exchangeable constraints can be enforced
on the orbit corresponding to this cluster. The log-partition upper
bounds of lifted TRW with different outer bounds are shown in Fig.
\ref{fig:ppi-logz}. The left part shows the region with negative
$\theta_{011}$ and the right part shows the region with positive
$\theta_{011}$ ($\theta_{111}$ is held fixed at $-0.021$). For
negative $\theta_{011}$, a combination of cycle and exchangeable
constraints (CE) is crucial to improve the upper bound. For positive
$\theta_{011}$, exchange constraints (LE) are already sufficient
to yield the best bound.

\textbf{Equitable Partition and the Spanning Tree Polytope}

\begin{figure}
\centering{}\includegraphics[scale=0.4]{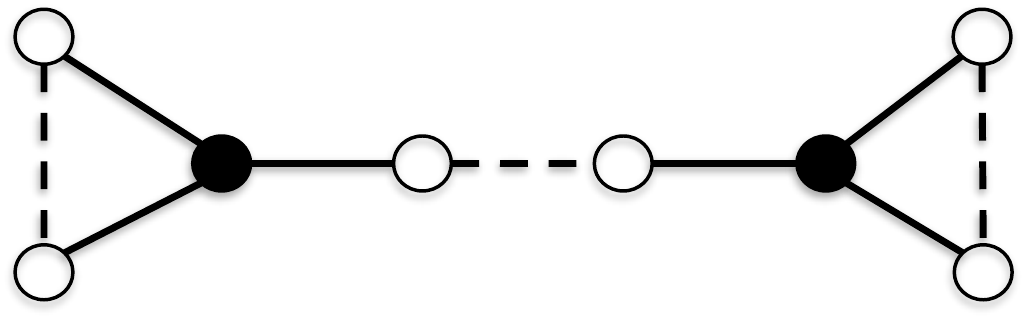}\caption{\label{fig:equitable-partition}$\gm$ and its equitable partition.
No element of the spanning tree polytope is uniform in the same cell
of the equitable partition.}
\end{figure}

\begin{example}
\label{example:equitable-stp}We give an example where the symmetrized
spanning tree polytope corresponding to the (edge) equitable partition
$\varepsilon$, $\symsub{\stp}{\epsilon}=\stp(\gm)\cap\Real_{[\varepsilon]}^{|E|}$
is an empty set. The example graph $\gm$ is shown in Fig. \ref{fig:equitable-partition}.
There are two edge types in the coarsest equitable partition: solid
and dashed. The dashed edge in the middle is a bridge, it must appear
in every spanning tree, so every $\ea\in\stp$ must assign $1$ to
this edge. If $\symsub{\stp}{\epsilon}\neq\emptyset$ then there is
some $\rho\in\symsub{\stp}{\epsilon}$ that assigns the same weights
to all dashed edges. Therefore it will assign 1 to the two dashed
edges on the left and right hand side. The remaining solid edges have
equal weight, and since the total weight is $|V|-1=7$, the solid
weight is $(7-3)/6=2/3$. Now consider the triangle on the left. This
triangle has the total weight $1+4/3>2$ which violates the constraint
of the spanning tree polytope. Thus, for this graph $\symsub{\stp}{\epsilon}$
must be empty.\end{example}

\end{document}